\documentclass[manuscript, screen]{acmart}

\usepackage[ruled,vlined]{algorithm2e}
\usepackage{subcaption}
\usepackage{cleveref}
\usepackage{multirow}
\usepackage{prettyref}
\newcommand{\ie}{\textit{i.e.}}
\newcommand{\eg}{\textit{e.g.}}
\newtheorem{theorem}{Theorem}
\usepackage{caption}
\usepackage{xcolor}

\begin{document}

\title{Building Real-time Awareness of Out-of-distribution in Trajectory Prediction for Autonomous Vehicles}

\author{Tongfei Guo} 
\email{guo.t@northeastern.edu}
\orcid{0009-0006-7949-5238}
\affiliation{%
  \institution{Northeastern University}
  \city{Boston}
  \state{Massachusetts}
  \country{USA}
}

\author{Taposh Banerjee}
\affiliation{%
  \institution{University of Pittsburgh}
  \city{Pittsburgh}
  \country{USA}}
\email{taposh.banerjee@pitt.edu}

\author{Rui Liu}
\affiliation{%
  \institution{Kent State University}
  \city{Kent}
  \country{USA}
}
\email{rliu11@kent.edu}

\author{Lili Su}
\affiliation{%
  \institution{Northeastern University}
  \city{Boston}
  \state{Massachusetts}
  \country{USA}
}
\email{l.su@northeastern.edu}

\renewcommand{\shortauthors}{Guo et al.}

\begin{teaserfigure}
    \centering
    \captionsetup{labelformat=empty}\includegraphics[width=0.85\textwidth]{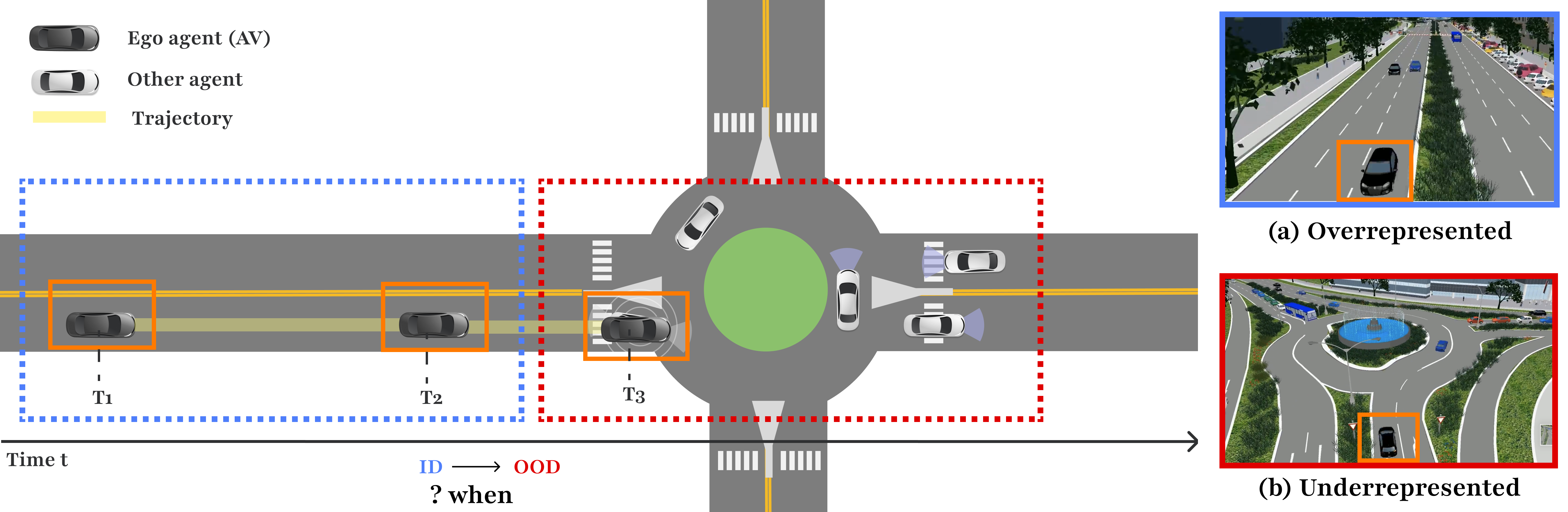}
    \caption{  Fig. 1.  An AV enters from overrepresented straight lanes (ID scene) to an underrepresented complex intersection (OOD scene) at time $t$. As the AV is not familiar with the intersection, it may make wrong trajectory predictions on neighboring vehicles, resulting in collisions. If we can infer the time $t$ that AV enters OOD scenes, we can pass back the controls to human drivers in time. }
    \Description{}
    \label{fig:fig1}
\end{teaserfigure}

\begin{abstract} 

Accurate trajectory prediction is essential for the safe operation of autonomous vehicles in real-world environments. Even well-trained machine learning models may produce unreliable predictions due to discrepancies between training data and real-world conditions encountered during inference. 
In particular, the training dataset 
tends to overrepresent common scenes (\eg, straight lanes) while underrepresenting less frequent ones (\eg, traffic circles). In addition, it often overlooks 
unpredictable real-world events such as sudden braking or falling objects. To ensure safety, it is critical to detect in real-time when a model’s predictions become unreliable. 
Leveraging the intuition that in-distribution (ID) scenes exhibit error patterns similar to training data, while out-of-distribution (OOD) scenes do not, we introduce a principled, real-time approach for OOD detection by framing it as a change-point detection problem. 
We address the challenging settings where the OOD scenes are deceptive, meaning that they are not easily detectable by human intuitions.   
Our lightweight solutions can handle the occurrence of OOD at any time during trajectory prediction inference.  
Experimental results on multiple real-world datasets using a benchmark trajectory prediction model demonstrate the effectiveness of our methods. 
\end{abstract}

\ccsdesc[500]{Computing methodologies ~ Sequential decision making }
\ccsdesc[500]{Mathematics of computing ~ Probability and statistics}

\keywords{autonomous vehicles, trajectory prediction, machine learning, out-of-distribution detection}

\maketitle
\markboth{Tongfei Guo et al.}{Building Real-time Awareness of Out-of-distribution in Trajectory Prediction for Autonomous Vehicles}

\section{INTRODUCTION}
\label{Introduction}

AI technologies are the backbone of modern autonomous vehicles (AVs), and are reshaping transportation systems. 
Accurate trajectory prediction is essential for the safe operation of autonomous vehicles in real-world environments. 
However, even well-trained machine learning (ML) models may produce unreliable predictions due to discrepancies between training data and real-world conditions encountered during inference.  
Specifically, many public datasets overrepresent certain driving scenes (\eg\,straight lanes and high-ways) while significantly underrepresenting others (\eg\,complex traffic circles or intersections), as illustrated in~Fig. \ref{fig:fig1}. Moreover, real-world environments suffer a wide range of uncertainties such as 
sudden braking by nearby vehicles, large objects falling from other vehicles, or other deceptive yet life-threatening distribution shifts~\cite{ filos2020can, tang2020naturalistic,zhang2022adversarial}. 
As a consequence, an AV may unexpectedly run into driving scenarios that are poorly represented by the training data, which we refer to as {\em out-of-distribution (OOD) data}. 
To ensure safety, it is critical to detect in real-time when a ML model’s trajectory predictions become unreliable, allowing control to be seamlessly transferred back to human drivers.  

\begin{figure}
\centering
\captionsetup{labelformat=empty} 
\includegraphics[width=0.7\linewidth]{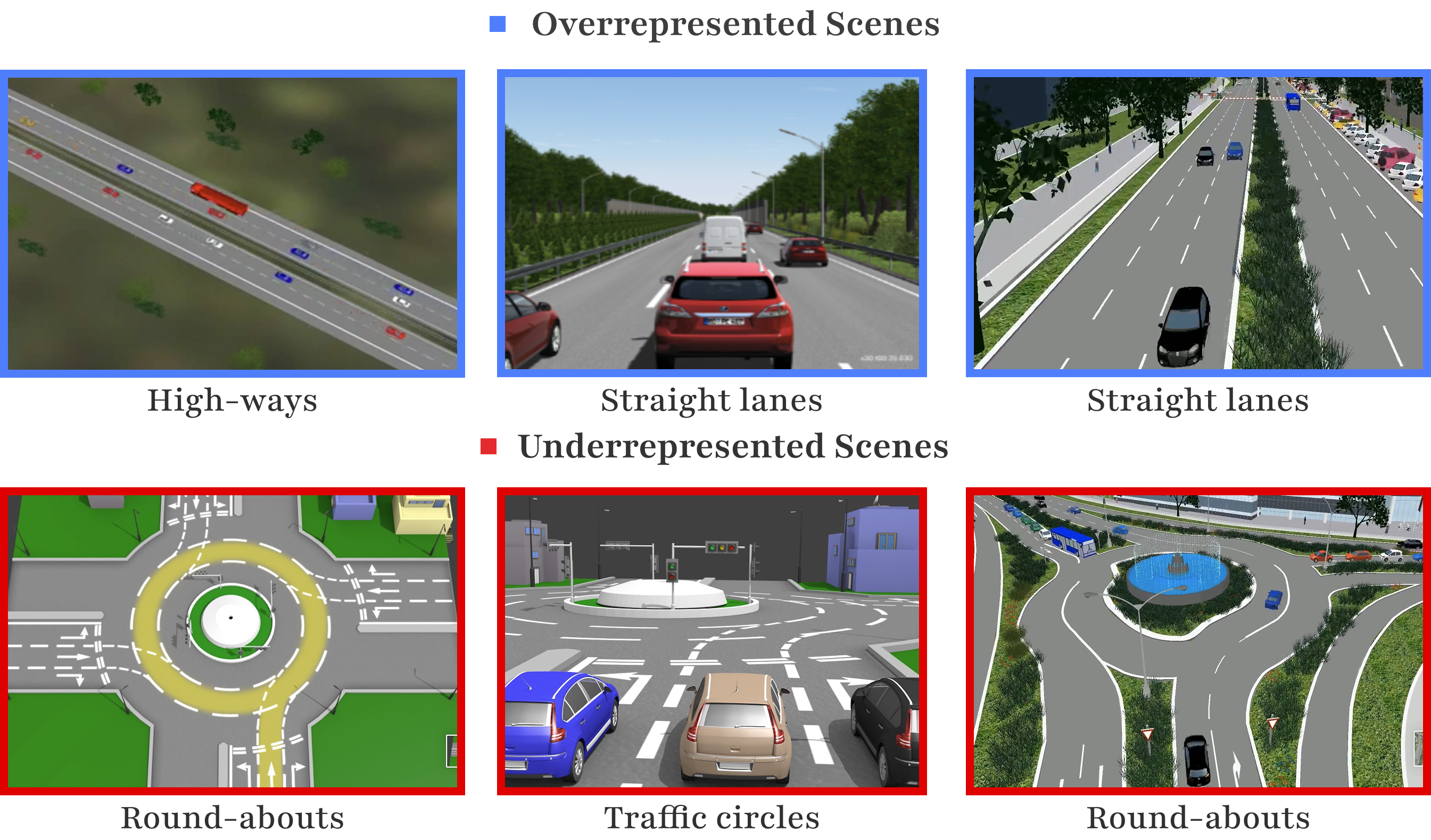}
\caption{  {Fig. 2. } Illustration of overrepresented and underrepresented scenes for the deployed ML models. Images from PTV Vissim simulation.}
\Description{}
\label{fig:fig2}
\end{figure}

Prior works focus on improving trajectory prediction reliability through anomaly detection~\cite{8684317} and uncertainty estimation~\cite{9815528}. However, these methods typically rely on the assumption that the training data fully represent real-world scenarios. Anomaly detection methods, such as RNN-based models, require fully representative training data to learn sequential dependencies and detect unusual trajectory patterns, such as sudden direction changes or irregular speeds~\cite{song2018anomalous}. Similarly, uncertainty estimation methods like variational autoencoders (VAEs) rely on comprehensive training data to model reliable distributions and flag anomalies via reconstruction errors~\cite{8794282}. Insufficient data coverage hinders their ability to quantify uncertainty in novel situations, degrading performance.

Some OOD scenes are easy to detect, such as sudden acceleration or rapid lane changes across multiple lanes. In this paper, we focus on a more challenging and practically relevant setting where OOD scenes are \textbf{deceptive}—meaning they are not easily identifiable by human intuition. Specifically, we examine subtle changes in the driving behavior of neighboring vehicles that maintain motion patterns within expected variations, making them difficult to distinguish from normal behavior. Building on the intuition that in-distribution (ID) driving scenes exhibit prediction error distributions similar to those observed during training, while OOD scenes deviate from this pattern, we closely monitor the sequence of prediction errors. We frame the trajectory prediction OOD awareness problem as a quickest change-point detection (QCD) problem.
Drawing on well-established techniques of sequential analysis, we develop principled and lightweight OOD detection methods that offer formally assured performance in the trade-off between detection delay and false alarm.

\noindent {\bf Contributions.} We are the first to apply QCD methods to detect OOD on multiple real-world trajectory prediction datasets. Our solutions are lightweight (monitoring a scalar variable of the prediction errors only) and can handle the occurrence of OOD at any time during inference. We summarize our main contributions as follows: 

\begin{itemize}
    \item We observe that pre- and post-change trajectory prediction errors are well-modeled as Gaussian Mixture Models, forming the basis of our OOD detection framework.
    \item We adapt the CUSUM algorithm for OOD detection, achieving 95\% faster detection with minimal false alarms compared to classical sequential methods.
    \item We test robustness across varying levels of distribution knowledge (complete, partial, unknown), showing our method remains the most effective in challenging scenarios.
\end{itemize}

\section{RELATED WORK}
\label{sec: related work} 
\subsection{Trajectory Prediction}

Safe autonomous driving in dynamic environments requires predicting future states of nearby traffic agents, particularly surrounding vehicles. Trajectory prediction methods for AVs are categorized into four types: physics-based, classic machine learning, deep learning, and reinforcement learning approaches.
Physics-based methods use kinematic~\cite{anderson2021kinematic, lytrivis2008cooperative} or dynamic~\cite{lin2000vehicle, pepy2006reducing, kaempchen2009situation} models. These approaches are simple and computationally efficient but struggle with long-term predictions in static environments, as they ignore complex interactions and environmental factors.
Classic machine learning methods, such as Gaussian Processes~\cite{wang2007gaussian}, Hidden Markov Models~\cite{deo2018would}, Dynamic Bayesian Networks~\cite{li2020pedestrian}, and Gaussian Mixture Models~\cite{mebrahtu2023transformer}, improve prediction accuracy and extend time horizons. However, their reliance on predefined maneuvers limits adaptability to diverse driving scenarios.
Deep learning architectures—including RNNs~\cite{jain2016structural}, CNNs~\cite{cui2019multimodal, phan2020covernet}, and GNNs~\cite{lee2017desire}—represent the state-of-the-art. These models excel in complex environments and long-term predictions by integrating interaction patterns and map data. Yet, they demand large training datasets and face trade-offs between computational efficiency and real-time performance.
Reinforcement learning methods, such as inverse reinforcement learning~\cite{gonzalez2016high, shen2018transferable} and imitation learning~\cite{qi2020imitative, choi2021trajgail}, mimic human decision-making to generate accurate long-horizon trajectories. However, they suffer from high computational costs and lengthy training processes.
Existing literature often assumes training data comprehensively covers all driving scenarios, a premise rarely valid in practice \cite{filos2020can}. Consequently, while current models perform well in ID scenes, their reliability degrades significantly in OOD scenarios. This limitation underscores the critical need for robust trajectory prediction to ensure AV safety in unseen dynamic environments.

\subsection{Out-of-Distribution Detection} 

OOD detection is critical for ensuring the reliability of ML models in open-world settings, where they frequently encounter scenarios not represented in the training data.
OOD data can emerge from a variety of sources, including environmental variability, unusual interactions, and intentional or rare events \cite{li2022coda}. For instance, environmental factors such as unexpected weather conditions (\eg, heavy rain, fog, or snow), abrupt lighting changes (\eg, transitions from daylight to nighttime), or unfamiliar road types (\eg, unpaved or construction zones) can introduce OOD scenes that were not accounted for during model training. Additionally, unpredictable behaviors from other road users—such as sudden stops, erratic lane changes, or jaywalking pedestrians—can further exacerbate the occurrence of OOD data. When trajectory prediction models encounter OOD data, their performance often degrades, leading to unreliable or unsafe predictions that can compromise the safety and efficacy of AV systems \cite{fang2022out, li2023rethinking}. 
Traditional OOD detection methods often rely on scoring functions that estimate the likelihood of a sample belonging to the training distribution, either through output-based approaches like maximum softmax probability~\cite{hendrycks2016baseline} and energy scores~\cite{liu2020energy}, or representation-based methods that perform density estimation in lower-dimensional latent spaces, often modeled as Gaussian~\cite{lee2018simple,ren2021simple} or using non-parametric techniques~\cite{sun2022out}. 
Despite their utility, these methods face challenges in high-dimensional domains like trajectory prediction. Density estimation becomes computationally intensive and struggles with real-time scalability, a critical requirement for low-latency AV systems. Additionally, no single approach consistently outperforms others across benchmarks \cite{tajwar2021no}, highlighting generalization limitations.
Motivated by these computational and performance constraints, we propose exploring efficient, domain-specific techniques—particularly sequential analysis—to enable real-time OOD detection in dynamic environments.

\subsection{Change-point Detection}
Change-point detection identifies abrupt shifts in time series data and has been extensively studied in statistics, signal processing, and machine learning, with applications in finance, bioinformatics, climatology, and network monitoring~\cite{  yang2006adaptive, malladi2013online, reeves2007review, itoh2010change,tartakovsky2012efficient}. The objective was to detect a shift in the mean of independent and identically distributed (iid) Gaussian variables for the purpose of industrial quality control.
A prominent branch of change point detection is quickest change detection (QCD), which focuses on detecting distributional shifts as rapidly as possible while controlling false alarms \cite{veeravalli2014quickest}. QCD is typically formulated under Bayesian and non-Bayesian frameworks. In the Bayesian setting, where the change point is treated as a random variable with a known prior, the Shiryaev test provides an optimal solution by declaring a change when the posterior probability exceeds a threshold \cite{shiryaev1963optimum, tartakovsky2005general}. Recent findings on QCD within the Bayesian framework are available in~\cite{banerjee2015data, tartakovsky2014sequential, guo2023bayesian, hou2024robust}.
In the non-Bayesian framework, the change point is unknown, and the objective is to minimize the worst-case detection delay, leading to the minimax-optimal Cumulative Sum (CUSUM) algorithm \cite{page1954continuous, lorden1971procedures, moustakides1986optimal}. CUSUM sequentially accumulates deviations from an expected baseline and triggers an alarm when a predefined threshold is crossed, making it first-order asymptotically optimal \cite{moustakides1986optimal, veeravalli2014quickest}. We believe that CUSUM is naturally well-suited for OOD detection due to its ability to identify deviations in real-time prediction errors. Its effectiveness has led to its widespread adoption in time-sensitive domains, as evidenced by prior research \cite{chowdhury2012bayesian, tartakovsky2006detection, nikiforov1995generalized}.

\section{PROBLEM FORMULATION}

\begin{table}[ht]
    \centering
    \renewcommand{\arraystretch}{1.2} 
    \label{tab: tab1}
    \caption{Table 1. Main notations used throughout the paper.}
    \begin{tabular}{c|p{12cm}}
        \hline
        \text{Notation} & \text{Description} \\
        \hline
        $\mathcal{D}$ & Dataset \\
        $N$ & Number of driving scenes \\ 
        $S_j$ & A single driving scene described by the triple $S_j = (\mathcal{X}_j, \mathcal{Y}_j, \mathcal{M})$ \\ 
        $\mathcal{X}_j$ & Collection of observed trajectories in scene $S_j$, $\mathcal{X}_j = \{x_{j}^1, \cdots, x_{j}^{m_{j}}\}$ \\
        $\mathcal{Y}_j$ & Collection of future trajectories to be predicted based on $\mathcal{X}_j$, $\mathcal{Y}_j = \{y_{j}^1, \cdots, y_{j}^{m_{j}}\}$ \\
        $\mathcal{M}$ & Map information; if unavailable, $\mathcal{M} = \emptyset$ \\
        $m_j$ & Number of agents (vehicles, pedestrians, etc.) in scene $S_j$ \\
        \hline
        $L_I$ & Number of time frames in the history trajectory \\
        $L_O$ & Number of time frames in the future trajectory \\
        \hline
        $\ell(\cdot)$ & Loss function \\
        $\lambda$ & Regularization coefficient \\ 
        \hline
        $\hat{h}$ & A given deployed ML model\\
        \hline
        $\epsilon_t$ & Observed prediction error at time $t$ \\
        $g_{\theta}(\epsilon_t)$ & Post-change distribution parameterized by $\theta$ \\
        $f_{\phi}(\epsilon_t)$ & Pre-change distribution parameterized by $\phi$  \\
        $\gamma$ & Change-point at which the distribution shifts from $f_{\phi}$ to $g_{\theta}$. \\        
        $\tau$ & Stopping time at which a change is detected \\
        $b$ & Detection threshold that controls sensitivity to change detection \\
        $W_t$ & CUSUM statistic at time $t$\\
        \hline
        $\eta$ & True gap between the post-change and pre-change distributions \\
        $\kappa$ & Assumed shift value used to evaluate robustness \\
        $\mathcal{P}_1$ & 
        The set of all possible $\hat{g}_{\theta}$ obtained by shifting the $f_{\phi}(\epsilon_t)$, $\mathcal{P}_1 = \{f_{\phi}(\epsilon_t - \eta) \mid \eta \geq \kappa\}$ \\
        $\mathcal{P}_2$ & 
        The convex hull of $\mathcal{P}_1$,  $\mathcal{P}_2 = \operatorname{conv}(\mathcal{P}_1)$ \\ 
        \hline
    \end{tabular}
    \label{tab:notations}
\end{table}

\subsection{Trajectory Representation and the Prediction Problem}  
\label{sec: sec3.1}
\noindent The training dataset $\mathcal{D} = \{S_j\}_{j=1}^N$ is a collection of driving scenes. 
Each $S_{j}$ described by a triple $S_{j} = (\mathcal{X}_j, \mathcal{Y}_j, \mathcal{M})$, where $\mathcal{X}_j$ is the collections of observed trajectories,  $\mathcal{Y}_j$ is the collection of future trajectories that we aim to predict based on $\mathcal{X}_j$, 
and $\mathcal{M}$ is the map information. When a dataset does not have map information, we set $\mathcal{M} = \emptyset$. 
In a driving scene, an agent represents a moving object such as a vehicle, a motorcycle, or a pedestrian.  
For a given driving scene $S_{j}$, let $m_{j}$ denote the number of agents in the scene. 
Hence, $\mathcal{X}_j$ and $\mathcal{Y}_j$ can be explicitly written out as 
$\mathcal{X}_j = \{x_{j}^1, \cdots, x_{j}^{m_{j}}\}$ and 
$\mathcal{Y}_j = \{y_{j}^1, \cdots, y_{j}^{m_{j}}\}$, 
where $x_{j}^i \in \mathbb{R}^{2\times L_I}$ and $y_{j}^i \in \mathbb{R}^{2\times L_O}$ 
for each agent $i$ in the scene. 
Here, $L_I$ and $L_O$ are the numbers of time frames in the observed trajectory $x_{j}^i$ and in the future trajectory $y_{j}^i$, respectively. It is worth noting that: \\  
(1) Different from traditional supervised learning problem, in the trajectory prediction problem, the trajectory of an agent is a sequence of positions, resulting in possibly often overlapping  $\mathcal{X}_j$ and $\mathcal{Y}_j$. \\ 
(2) Different driving scenes may contain different numbers of agents as can be seen in Fig.\,\ref{fig:fig1}. 

A machine-learning model $\hat{h}$ for the trajectory prediction task can be trained by minimizing the objective function that may take the form  
\[
\min_f \quad 
\frac{1}{N} \sum_{j=1}^N \ell\left(\hat{h}(\mathcal{X}_j), \mathcal{Y}_j \right) + \lambda \times \text{regularization}, 
\]
where $\ell$ is a loss function such as the negative log-likelihood (NLL) of the Laplace/Gaussian distributions \cite{peng2023privacy,tang2024hpnet}, $\lambda\ge 0$ is the regularization coefficient, and one popular choice of the regularization term in the objective is the cross entropy loss.

The specific training methods are out of the scope of this paper. Readers can find concrete instances of the objective functions in \cite{peng2023privacy,tang2024hpnet} and others.  
Notably, different from traditional machine learning tasks, where the model trained has fixed input and output dimensions, the trajectory prediction model \cite{tang2024hpnet} $\hat{h}$ may have varying input and output dimensions.

\vskip \baselineskip 

\noindent {\bf Evaluation Metrics.}
Let $\hat{h}$ be a given deployed trajectory prediction model. 
We follow existing literature on trajectory prediction model training to evaluate the generalization performance of $\hat{h}$ \cite{wang2019exploring, dendorfer2021mg}.
Let $\mathcal{D}_{\text{test}} = \{\mathcal{S}_j\}_{j=1}^{N_{\text{test}}}$ be the test dataset. 
Instead of evaluating the trajectory prediction errors of all neighboring agents, only the prediction errors on an arbitrarily chosen agent will be evaluated. Such chosen agent is referred to as {\em target agent} of scene $\mathcal{S}_j$, denoted as $t_j$. 
We use three common metrics in trajectory prediction: {\em Average displacement error} (ADE) and {\em Final displacement error} (FDE) and {\em Root Mean Squared Error} (RMSE). These metrics are widely adopted in prior works \cite{kim2020multi, deo2018convolutional}, including the ICRA 2020 \textit{nuScenes} prediction challenge \cite{ivanovic2022injecting, liu2024laformer}.  

- \textbf{ADE} measures the average error over the entire prediction horizon:
  \[
  \text{ADE} ~ := ~ \frac{1}{N_{\text{test}}}\sum_{j=1}^{N_{\text{test}}} \frac{1}{L_O} \| \hat{h}(x_{j}^{t_j}) -  y_{j}^{t_j}\|_2,
  \]
  
- \textbf{FDE} captures the error at the final time step of the predicted trajectory:
  \[
  \text{FDE} ~ := ~ \frac{1}{N_{\text{test}}}\sum_{j=1}^{N_{\text{test}}} \frac{1}{L_O} \| [\hat{h}(x_{j}^{t_j})]_{L_O} -  [y_{j}^{t_j}]_{L_O}\|_2, 
  \]
where $[\hat{h}(x_{j}^{t_j})]_{L_P}$ and $[y_{j}^{t_j}]_{L_P}$ are the last positions in $\hat{h}(x_{j}^{t_j})$ and $y_{j}^{t_j}$, respectively.

- \textbf{RMSE} computes the square root of the average squared differences between predicted and true positions over the prediction window:
  \[
  \text{RMSE}_t ~ := ~ \frac{1}{N_{\text{test}}}\sum_{j=1}^{N_{\text{test}}} \sqrt{\frac{1}{L_O} \| \hat{h}(x_{j}^{t_j}) -  y_{j}^{t_j}\|^2_2}.  
  \]

\subsection{Casting OOD Awareness as a QCD Problem}
\label{sec: sequential_analysis}

In the context of trajectory prediction, we define OOD awareness as the ability to detect deviations in prediction performance over time. Specifically, we consider a sequence of prediction errors $\epsilon_t$ arising from a 
trajectory prediction model $\hat{h}$. The sequence of errors is defined as
\begin{equation}
    \epsilon_t = d(\hat{h}(\mathcal{X}_t), \mathcal{Y}_t),
\end{equation}
where $d(\cdot, \cdot)$ is a distance function measuring the discrepancy between the predicted future trajectory and the ground truth. As mentioned in  Section \ref{sec: sec3.1}, we consider three distance functions, \ie, ADE, FDE, and RMSE. 

The AV observes the sequence $\{\epsilon_t\} = \{\epsilon_1, \epsilon_2, \dots, \epsilon_t\}$ of the target agent ($a$) for every scene. 
Let the data distribution change at an unknown time $\gamma$.
We denote the pre-change distribution of $\epsilon_t$ by $f_{\phi}(\epsilon_t)$ for $t < \gamma$ and the post-change distribution by $g_{\theta}(\epsilon_t)$ for $t \geq \gamma$, \ie, 
\begin{equation}
    \epsilon_t \sim \begin{cases}
        f_\phi(\epsilon_t), & t < \gamma \text{ (ID Scene)}, \\
        g_\theta(\epsilon_t), & t \geq \gamma \text{ (OOD Scene)}.
    \end{cases}
\end{equation}
Let $\mathcal{A}$ be any algorithm. Let $\tau_{\mathcal{A}}$ be the time that a change is declared under algorithm $\mathcal{A}$. When the adopted $\mathcal{A}$ is clear from the context, we drop the subscript $\mathcal{A}$ in $\tau$. Two key metrics for evaluating OOD detection algorithms are: (1) \emph{Detection Delay}—time to detect an OOD event, reflecting responsiveness; and (2) \emph{False Alarm Rate}—frequency of incorrect OOD flags, ensuring reliability. These metrics are widely used in \cite{lai2008quickest, lau2018binning, liang2024quickest} due to their ability to capture the trade-off between sensitivity and robustness.

For any $\gamma < \infty$, a true detection happens if $\tau \geq \gamma$ and false if $\tau < \gamma$. The design of the quickest change detection procedures involves optimizing the tradeoff between the delay to detection $(\tau - \gamma + 1)^+$ and false alarm. 

To ensure robust safety guarantees, we follow Lorden's minimax formulation \cite{lorden1971procedures}, which quantifies the detection delay using the Worst-Case Average Detection Delay (WADD), 
\begin{equation}
    \text{WADD}(\tau) = \sup_{t \geq 1} \text{ess} \sup \, \mathbb{E}_t[(\tau - \gamma + 1 )^+ \mid \epsilon_1, \dots, \epsilon_{t-1}],
\label{eq: delay}
\end{equation} 
where \((\cdot)^+ = \max\{0, \cdot\}\), \(\mathbb{E}_t[\cdot]\) denotes the expectation when the change occurs at time \(t\), and \(\text{ess} \sup(\cdot)\) refers to the essential supremum of a scalar random variable.

The metric of false alarm can be measured by the false alarm rate $\text{FAR}(\tau)$ \cite{veeravalli2014quickest} or by the mean time to false alarm MTFA$(\tau)$, \ie,  
\begin{equation}
    \text{FAR}(\tau) = \frac{1}{\mathbb{E}_\infty(\tau)} \quad  \text{and}  \quad
    \text{MTFA}(\tau) = \mathbf{E}_{\infty}[\tau \Pi_{\tau < \gamma}].
\label{eq: far}
\end{equation}
where \(\mathbb{E}_\infty\) denotes the expectation 
when the change never occurs.

\section{PRELIMINARIES}
\label{sec: preliminaries}
This section provides the preliminaries for our OOD awareness algorithms.
In Section \ref{subsec:uncertainty-generation}, we outline the high-level approach for generating deceptive OOD scenes in our experiments.
In Section \ref{subsec: Gaussian mixture experiment}, we present preliminary experimental observations on the statistical structure of the pre- and post-change distributions.

\subsection{Deceptive OOD Scene Generation}  
\label{subsec:uncertainty-generation}

Deceptive OOD scenes refer to driving scenarios where deviations from expected behavior are subtle and difficult to identify, even for human observers. Unlike easily detectable OOD events such as sudden acceleration or erratic lane changes, deceptive OOD scenes maintain motion patterns that fall within normal variations, making them particularly challenging to recognize. These scenarios can arise from minor trajectory shifts of neighboring vehicles due to environmental factors such as road debris, unexpected pedestrian movements, or subtle driver behaviors. 

To systematically generate deceptive OOD scenes, we adopt the adversarial perturbation framework proposed by \cite{zhang2022adversarial}. 
Fig. \ref{fig: fig3} illustrates the process of introducing controlled modifications to historical trajectory data while ensuring the altered trajectories remain physically plausible. In particular, we use the single-frame perturbation method in \cite{zhang2022adversarial}, \ie, perturbing only $L_P=1$ time-frame among the $L_I$ history trajectory time-frames.  
%
%
Fig. \ref{fig:fig4} provides a comparative visualization of AV trajectory prediction performance in normal (ID) and perturbed (OOD) scene, highlighting why deceptive OOD cases pose a greater hazard than one might expect. In the ID setting (left), the ego vehicle (gray car) accurately predicts the trajectories of surrounding vehicles, maintaining safe driving behavior. However, in the OOD setting (right), the target vehicle undergoes a subtle deviation due to an external factor, such as an obstacle ahead. Although the deviation is minor, it causes the ego vehicle's ML model to mispredict the target's future path, leading to an unnecessary sudden brake maneuver. This reaction could result in a rear-end collision. 
\begin{figure}
\centering
\captionsetup{labelformat=empty} 
\includegraphics[width=0.5\textwidth]{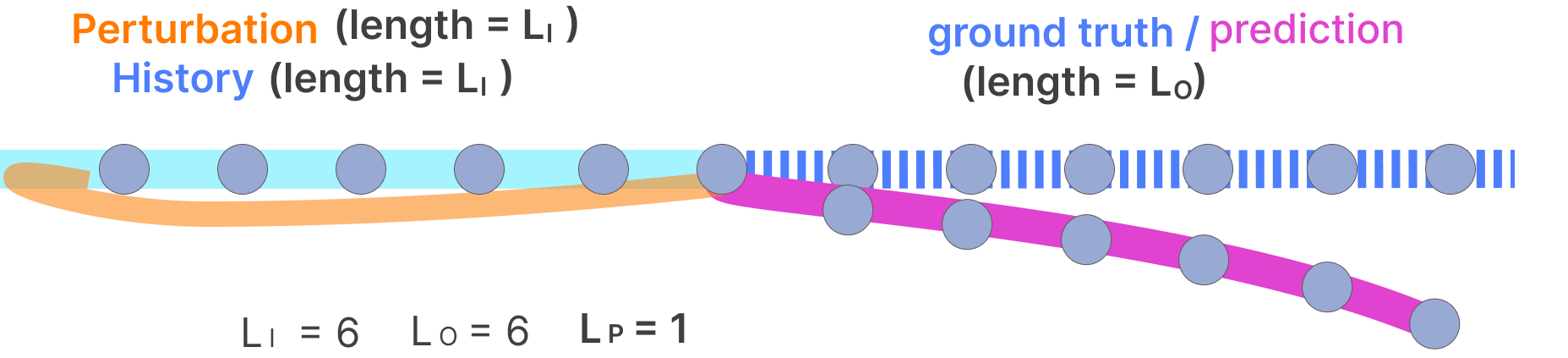}
\caption{  {Fig. 3.} Visualization of minor perturbation applied to the history trajectory of the target vehicle. The input and output trajectories consist of \( L_I = 6 \) and \( L_O = 6 \) time frames, respectively, while \( L_P = 1 \) denotes the perturbed time frame within the input. Dots represent observed and predicted trajectory positions at different time frames.  }
\Description{}
\vskip -.8ex
\label{fig: fig3}
\vskip -1.2ex
\end{figure}

\begin{figure}
  \centering
  \captionsetup{labelformat=empty} 
  \includegraphics[width=0.8\textwidth]{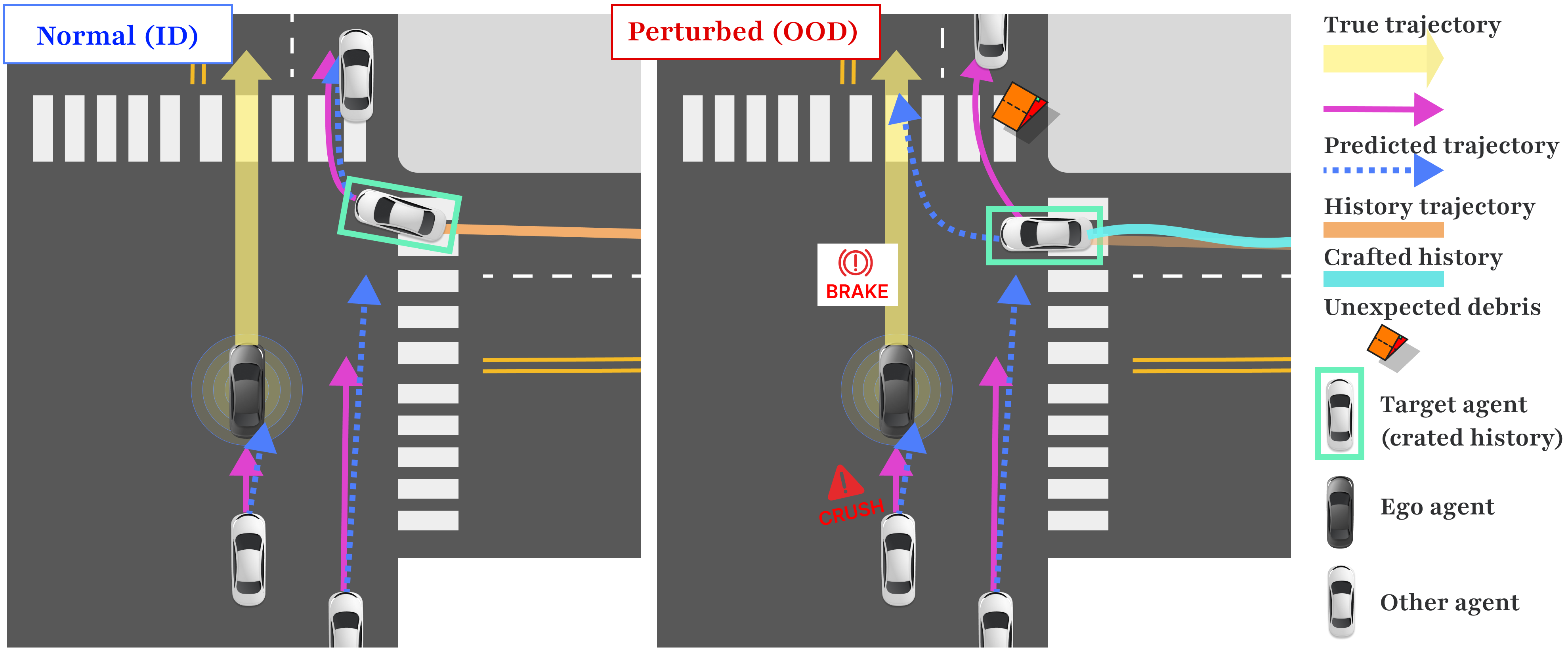}
  \caption{  {Fig. 4.} Illustration of the performance comparison of a given ML trajectory prediction model on ID and OOD scenes. In an ID scenario (left figure), the ego vehicle (gray car), on which the ML model is implemented, can accurately predict the trajectories of neighboring vehicles. In an OOD scenario (right figure), the target vehicle slightly deviates from its usual driving path due to unexpected debris from the vehicle ahead. As a result, the ego vehicle may mispredict the target vehicle's trajectory, mistakenly assuming it will move into the same lane. In response, the ego vehicle may suddenly brake, potentially causing a rear-end collision.}
  \Description{}
  \label{fig:fig4}
\end{figure}

\subsection{Gaussian Mixture of Pre- and Post-Change Distributions}
\label{subsec: Gaussian mixture experiment}
\begin{figure}[!t] 
\captionsetup{labelformat=empty}
    \centering
    \subfloat[GRIP++ ]{\includegraphics[width=.35\linewidth]{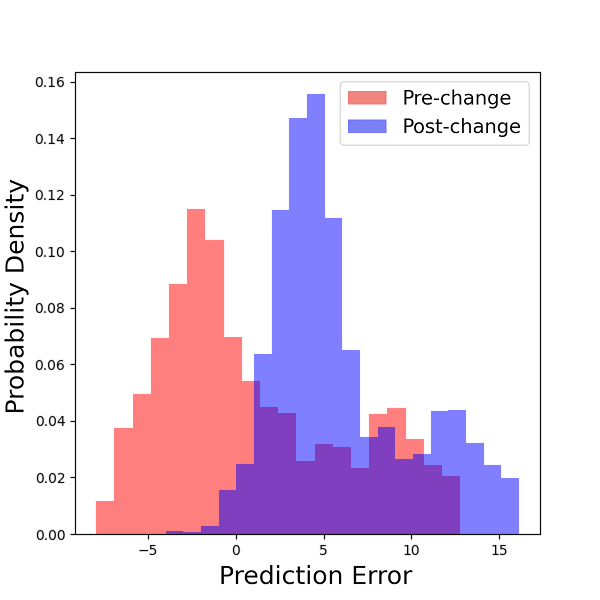}    
    \label{fig:fig5a}}
    \subfloat[FQA ]{\includegraphics[width=.35\linewidth]{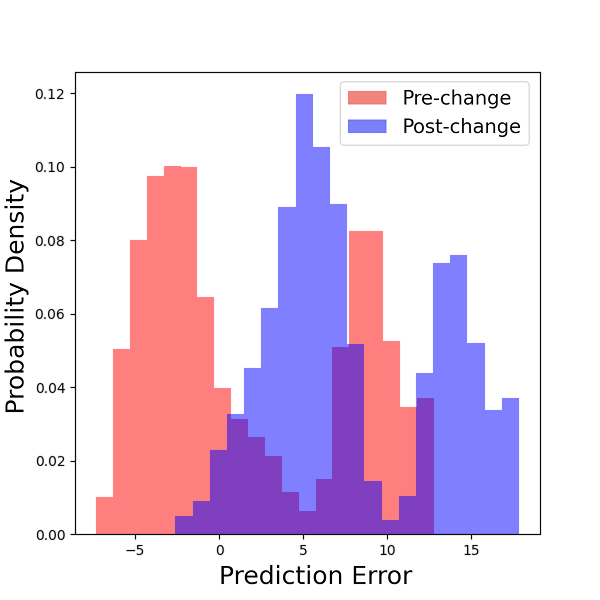}    
    \label{fig:fig5b}}
\caption{  {Fig. 5. } Illustration of the mixture Gaussian distributions for pre-change (red) and post-change (blue) error distributions in trajectory prediction, assessed using the ADE metric on the ApolloScape dataset. (a) Results from the GRIP++ model. (b) Results from the FQA model. Both examples highlight the observed shift in error distributions.}
\Description{}
\label{fig:fig5}
\end{figure}

Understanding the statistical structure of trajectory prediction errors is critical for detecting meaningful changes in autonomous driving systems. To address this, we analyze the pre- and post-change distributions of these errors using real-world datasets. Our key observation is that the data is approximately represented by a \textit{Gaussian mixture model} (GMM). 
As illustrated in Fig. \ref{fig:fig5}, the GMM fits for pre- and post-change ADE distributions in the ApolloScape dataset. Fig. \ref{fig:fig5a} presents results from the GRIP++ model, while Fig. \ref{fig:fig5b} depicts outcomes from the FQA model. Similar trends were consistently observed across other datasets (\eg, NGSIM and nuScenes) and for additional metrics (\eg, FDE and RMSE), though these results are omitted due to space constraints.

\section{METHODS}
\label{methods}
In this section, we present the key methods used for detecting OOD scenes, focusing on CUSUM as the primary algorithm, with Z-Score and Chi-Square tests serving as benchmark comparisons. CUSUM is a mature method known for its robust theoretical foundation and ability to quickly and reliably detect distributional shifts \cite{manogaran2018spatial, sun2022quickest}. In particular, CUSUM is designed to minimize detection delay while keeping the false alarm rate within acceptable bounds, providing formal performance guarantees in dynamic environments. As benchmarks, the Z-Score and Chi-Square methods are commonly used in literature for comparison \cite{cheadle2003analysis, mare2017nonstationary, wallis1941significance, chiang2024time}.

\subsection{Background of CUSUM}
\label{subsec: cusum background}
As discussed in Section \ref{sec: sequential_analysis}, we denote the unknown change point by $\gamma$, while $t_0$ represents the current system time, incrementing by 1 at each time step. The problem of detecting OOD can be formulated as detecting  a shift in the distribution of the prediction errors $\{\epsilon_t\}$. As shown in~Fig. \ref{fig: fig3}, we observe that on some popular datasets of trajectory prediction, the pre- and post-change distributions are Gaussian mixtures with two components. We evaluate the CUSUM algorithm under three common scenarios (Section \ref{sec: full-knowldge}, Section \ref{subsec: partial knowledge}, and Section \ref{subsec: unknown}).

The CUSUM algorithm detects distributional shifts by monitoring cumulative deviations in the log-likelihood ratio between pre-change (\(f_{\phi}\)) and post-change (\(g_{\theta}\)) distributions. The decision statistic \(W_t\) is computed iteratively 
\begin{equation}
W_t = 
\begin{cases} 
0, & \text{if } t = 0, \\
\max (W_{t-1} + \log L(\epsilon_t), 0), & \text{if } t \geq 1,
\label{eq:wt}
\end{cases}
\end{equation}
where \(L(\epsilon_t) = \frac{g_{\theta}(\epsilon_t)}{f_{\phi}(\epsilon_t)}\) is the likelihood ratio. A change is declared when \(W_t\) exceeds a threshold \(b\)
\[
\tau = \inf \{t \geq 1 : W_t \geq b\}.
\]

This method optimally balances the trade-off between false alarm rate $\text{FAR}(\tau < \gamma)$ and detection delay $\mathbb{E}[\tau - \gamma \mid \tau \geq \gamma]$. This method enables awareness of sequential decision reliability by continuously monitoring prediction errors.
For implementation details, see~Algorithm \ref{alg: cusum_detection}. 

\begin{theorem}[\cite{veeravalli2014quickest}]
\label{thm: cusum full}
The CUSUM test is optimal for minimizing $\textup{WADD}$ given in \eqref{eq: far} subject to any fixed constraint of $\alpha$ on $\textup{FAR}$ given in \eqref{eq: far}. In addition, 
for any \(\alpha \in (0, 1)\), setting the threshold \(b = |\log \alpha|\) ensures  
\[
\textup{FAR}(\tau) \leq \alpha \quad \textup{and WADD}(\tau) = \frac{|\log \alpha|}{D_{KL}(g_{\theta}, f_{\phi})} (1+o(1)),
\]
as $\alpha\to 0$, where \(D_{KL}(g_{\theta}, f_{\phi})\) is the Kullback-Leibler divergence between pre-change and post-change distributions.  
\end{theorem}

\noindent
\textbf{Rationale.} After a change-point \(\gamma\), the observations \(\epsilon_t\) are distributed according to \(g_{\theta}\), and the expected log-likelihood ratio \(\mathbb{E}_{t \geq \gamma}[\log(\frac{g_{\theta}(\epsilon_t)}{f_{\phi}(\epsilon_t)})]\) equals the Kullback-Leibler divergence \(\text{$D_{KL}$}(g_{\theta} \parallel f_{\phi})\), which is positive. Conversely, before \(\gamma\), the observations are governed by \(f\), and \(\mathbb{E}_{t < \gamma}[\log(\frac{g_{\theta}(\epsilon_t)}{f_{\phi}(\epsilon_t)})] = -\text{$D_{KL}$}(f_{\phi} \parallel g_{\theta})\), which is negative. Consequently, with an increasing number of observations, the cumulative sum \(\sum_{t : t \geq \gamma} \log(\frac{g_{\theta}(\epsilon_t)}{f_{\phi}(\epsilon_t)})\) is expected to exceed the threshold \(b\). The \(\max\{\cdot, 0\}\) function in the statistic \(W_t\) keeps the statistic from going to $-\infty$ when the change time is large. 
The threshold \(b\) controls the sensitivity of the detection; increasing \(b\) makes the algorithm less prone to false alarms but also delays the detection of actual changes. 

\begin{figure}[!t]
\begin{algorithm}[H]
\KwIn{Threshold $b > 0$; pre-change distribution $f_{\phi}(\epsilon_t)$; post-change distribution $g_{\theta}(\epsilon_t)$}
\KwOut{Change point $\tau$ if detected}

\BlankLine
Initialize $W_0 \gets 0$\;%

\BlankLine
\While{true}{
    get a new observation of the prediction error $\epsilon_t$\;
    \tcc{Compute the log-likelihood ratio}

    $L(\epsilon_t) =  \frac{g_{\theta}(\epsilon_t)}{f_{\phi}(\epsilon_t)}$

    $W_{t+1} \gets \max( W_{t} + \log L(\epsilon_t),0)$\;

     \If{$W_{t+1} \ge b$}{
        Declare a change point\;
        {\bf return} $\tau= t+1$\; 
    }
}

\caption{Ideal CUSUM with Parameterized Models}
\Description{}
\label{alg: cusum_detection} 
\end{algorithm} 
\end{figure}

\subsection{CUSUM for OOD Scene Detection}
\label{subsec: algoirthms}
In this subsection, we present several variants of the CUSUM tests, each differing in its requirements for prior knowledge of the pre- and post-change distributions. 
Section \ref{sec: full-knowldge} presents the ideal CUSUM test where both pre- and post-change distributions are fully known and provided as inputs to the algorithm.  
Section \ref{subsec: partial knowledge} describes two tests that require only partial knowledge of the pre- and/or post-change distributions.
Section \ref{subsec: unknown} introduces a test that does not require any knowledge on the post-change distribution. 

\subsubsection{CUSUM with Complete Knowledge. }
\label{sec: full-knowldge}
From Section \ref{subsec: Gaussian mixture experiment}, we know that  
both the pre-change (\( f_{\phi} \)) and post-change (\( g_{\theta} \)) distributions follow GMMs. The probability density function for GMM can be defined as  $ p(\epsilon_t) = \sum_{i=1}^{K} w_i \cdot \mathcal{N}(\epsilon_t \mid \mu_i, \sigma_i^2)$, where \( K \) is the number of Gaussian components in the mixture, and \( w_i \), \( \mu_i \), and \( \sigma_i^2 \), respectively, are the mixture weight, the mean, and the variance of component $i$.

In our experimental results in Fig.~\ref{fig: fig3}, we run Algorithm \ref{alg: cusum_detection} by plugging in the Gaussian mixture pdfs obtained from Section \ref{subsec: Gaussian mixture experiment} where $K=2$.  
Since this method is only a concrete instance of the ideal CUSUM, Theorem \ref{thm: cusum full} is applicable, \ie, this method is optimal. 

\subsubsection{CUSUM with Partial Knowledge. ~}
\label{subsec: partial knowledge}
In many real-world applications, post-change distributions are difficult to characterize precisely. Occasionally, in some cases, even pre-change distributions may be challenging to obtain.  In those cases, we need to approximate the post- or/and pre-change distributions. However, not all approximations are correct.  
We use the following definition to formalize the notion of correctness for OOD detection.
\begin{definition}
\label{def:approxCorr}
[Approximation correctness]
\label{def: correctness of cusum}
Let $\hat{L}(\epsilon_t)$ be the likelihood approximation used. Let $\hat{W}$ denote the CUSUM statistic defined as 
$\hat{W}_{t+1} = \max(\hat{W}_{t}+\log \hat{L}(\epsilon_t), ~0)$ with $\hat{W}_{0}=0$.  
We say the approximation $\hat{L}$ is correct if for 
any finite-time change-point $\gamma <\infty$, 
\begin{align*}
\mathbb{E}_{t < \gamma}[\log\hat{L}(\epsilon_t))]<0, ~~ \text{and} ~~~ 
\mathbb{E}_{t \geq \gamma}[\log\hat{L}(\epsilon_t))]> 0. 
\end{align*} 
\end{definition} 
Intuitively, $\mathbb{E}_{t < \gamma}[\log\hat{L}(\epsilon_t))]<0$
prevents an overwhelming number of false alarms before a change occurs, while 
$
\mathbb{E}_{t \geq \gamma}[\log\hat{L}(\epsilon_t))]> 0$
ensures that once a change occurs, it will be detected eventually.

\noindent \underline{\em Partial knowledge of post-change. ~}
The exact parameters of the post-change distribution are often unknown. Instead, we leverage partial knowledge or side information to approximate the post-change distributions. 
In particular, we consider the case when the overall mean $\mu_g$ and variance $\sigma^2_g$ of the post-change distribution are known. 
We use $\mathcal{N}(\epsilon_t \,|\,\mu_g, \sigma_g^2)$ -- a single Gaussian distribution with the given mean and variance -- to replace the true post-change distribution $g_{\theta}$ in computing the likelihood ratio in \eqref{eq:wt}, \ie, 
\begin{align*}
L(\epsilon_t) = \frac{\hat{g}_{\theta}(\epsilon_t)}{f_{\phi}(\epsilon_t)}, ~~~\text{where} ~   \hat{g}_{\theta}(\epsilon_t) = \mathcal{N}(\epsilon_t \,|\, \mu_g, \sigma_g^2). 
\end{align*}
Let $W_t^{p1}$ denote the CUSUM statistic with the above approximation, \ie, 
\begin{align}
\label{eq: cusum; partial; post1}
W_t^{p1} =  \max\left(0, ~ W_{t-1}^{p1} + \log \frac{\hat{g}_{\theta}(\epsilon_t)}{f_{\phi}(\epsilon_t)}\right), \quad W_0^{p1} = 0.
\end{align} 

\noindent{\bf Rationale.}
To ensure correctness as per Definition \ref{def: correctness of cusum}, 
one needs to check 
\begin{align*}
\mathbb{E}_{t \geq \gamma}[\log(\frac{\hat{g}_{\theta}(\epsilon_t)}{f_{\phi}(\epsilon_t)})]>0, ~~~ \text{and} ~ \mathbb{E}_{t < \gamma}[\log(\frac{\hat{g}_{\theta}(\epsilon_t)}{f_{\phi}(\epsilon_t)})]<0.  
\end{align*}
Note that 
$
\mathbb{E}_{t < \gamma}[\log(\frac{\hat{g}_{\theta}(\epsilon_t)}{f_{\phi}(\epsilon_t)})] = \int_{\epsilon} f_{\phi}(\epsilon) \frac{\hat{g}_{\theta}(\epsilon)}{f_{\phi}(\epsilon)} d\epsilon = - D_{KL}(f_{\phi}\| \hat{g}_{\theta})<0.    
$
In addition, it holds that 
\begin{align*}
\mathbb{E}_{t \geq \gamma}[\log(\frac{\hat{g}_{\theta}(\epsilon_t)}{f_{\phi}(\epsilon_t)})] 
& = \int_{\epsilon} g_{\theta}(\epsilon) \log \frac{\hat{g}_{\theta}(\epsilon)}{f_{\phi}(\epsilon)} d\epsilon
= \int_{\epsilon} g_{\theta}(\epsilon) \log \frac{g_{\theta}(\epsilon)}{f_{\phi}(\epsilon)} \frac{\hat{g}_{\theta}(\epsilon)}{g_{\theta}(\epsilon)} d\epsilon  \\
& = D_{KL}(g_{\theta}(\epsilon)\| f_{\phi}(\epsilon)) - D_{KL}(g_{\theta}(\epsilon)\| \hat{g}_{\theta}(\epsilon)). 
\end{align*}
In general, the sign of $\mathbb{E}_{t \geq \gamma}[\log(\frac{\hat{g}_{\theta}(\epsilon_t)}{f_{\phi}(\epsilon_t)})]$ is undetermined, and depends on the relative magnitudes of $D_{KL}(g_{\theta}(\epsilon)\| f_{\phi}(\epsilon))$ and  $D_{KL}(g_{\theta}(\epsilon)\| \hat{g}_{\theta}(\epsilon))$. Fortunately, as shown in Table \ref{tab:kl_gripp_apollo}, in our experiments, its sign is consistently positive across metrics and models, with full numerical analysis presented in Section \ref{sec:num_results}.

\begin{table}[htbp] 
\centering
\caption{  {Table 2.} Numerical evaluation of $\mathbb{E}_{t \geq \gamma}[\log(\frac{\hat{g}_{\theta}(\epsilon_t)}{f_{\phi}(\epsilon_t)})]$ under the FQA model, using ApolloScape dataset and ADE metric.}
\vspace{-0.3cm}
\begin{tabular}{|c||c|c|c|}
    \hline
    & \multicolumn{3}{c|}{
    \textbf{$\mathbb{E}_{t \geq \gamma}[\log(\frac{\hat{g}_{\theta}(\epsilon_t)}{f_{\phi}(\epsilon_t)})]$}} \\
    \hline 
    Components 
    & $D_{KL}(g_\theta\| f_{\phi})$
    & $D_{KL}(g_{\theta}\| \hat{g}_{\theta})$ 
    & sign of difference\\
    \hline
    Values & 0.2026 & 0.0405 & positive \\
    \hline
\end{tabular}
\label{tab:kl_gripp_apollo}
\end{table}


\noindent \underline{\em Partial knowledge of both pre- and post-change. ~}
Occasionally, the pre-change distribution may not be known, which may result from a lack of sufficient data to reconstruct the true distribution. 
We consider the case when the overall means ($\mu_f$ and $\mu_g$) and variances ($\sigma_f^2$ and $\sigma^2_g$) of pre- and post-change distributions are known. 
We use single Gaussians to approximate the pre- and post-change distributions in computing the likelihood ratio in \eqref{eq:wt}, \ie, 
\begin{align*}
L(\epsilon_t) = \frac{\hat{g}_{\theta}(\epsilon_t)}{\hat{f}_{\phi}(\epsilon_t)}, 
\end{align*}
where $\hat{f}_{\theta}(\epsilon_t) = \mathcal{N}(\epsilon_t \,|\, \mu_f, \sigma_f^2)$  
and $\hat{g}_{\theta}(\epsilon_t) = \mathcal{N}(\epsilon_t \,|\, \mu_g, \sigma_g^2)$. 
Let $W_t^{p2}$ denote the CUSUM statistic, \ie, 
\begin{align}
\label{eq: cusum; partial; post2}
W_t^{p2} =  \max\left(0, ~ W_{t-1}^{p2} + \log \frac{\hat{g}_{\theta}(\epsilon_t)}{\hat{f}_{\phi}(\epsilon_t)}\right), \quad W_0^{p2} = 0.
\end{align} 

\noindent{\bf Rationale.} 
Different from the case when pre-change distribution is known, here the signs of $\mathbb{E}_{t \geq \gamma}[\log(\frac{\hat{g}_{\theta}(\epsilon_t)}{\hat{f}_{\phi}(\epsilon_t)})]$ and $\mathbb{E}_{t < \gamma}[\log(\frac{\hat{g}_{\theta}(\epsilon_t)}{\hat{f}_{\phi}(\epsilon_t)})]$ are not generally predetermined. In particular, 
\begin{align*}
\mathbb{E}_{t \geq \gamma}[\log(\frac{\hat{g}_{\theta}(\epsilon_t)}{\hat{f}_{\phi}(\epsilon_t)})] = D_{KL}(g_{\theta}\| \hat{f_{\phi}}) - D_{KL}(g_{\theta}\|\hat{g}_{\theta}), 
~~ \text{and} ~~  
\mathbb{E}_{t < \gamma}[\log(\frac{\hat{g}_{\theta}(\epsilon_t)}{\hat{f}_{\phi}(\epsilon_t)})] 
= - D_{KL}(f_{\phi}\| \hat{g}_{\theta}) + D_{KL}(f_{\phi}\|\hat{f}_{\phi}). 
\end{align*}
Fortunately, in our experiments, the former is >0 and the latter <0, satisfying the correctness criterion in Definition \ref{def: correctness of cusum}.

\begin{table}[htbp] 
\centering
\caption{  {Table 3.} Numerical evaluation of $\mathbb{E}_{t <\gamma}[\log(\frac{\hat{g}_{\theta}(\epsilon_t)}{\hat{f}_{\phi}(\epsilon_t)})]$ and $\mathbb{E}_{t \geq \gamma}[\log(\frac{\hat{g}_{\theta}(\epsilon_t)}{\hat{f}_{\phi}(\epsilon_t)})]$ under the FQA model, using ApolloScape dataset and ADE metric.}
\vspace{-0.3cm}
\begin{tabular}{|c||c|c|c||c|c|c|}
    \hline
    & \multicolumn{3}{c||}{
    \textbf{$\mathbb{E}_{t < \gamma}[\log(\frac{\hat{g}_{\theta}(\epsilon_t)}{\hat{f}_{\phi}(\epsilon_t)})]$}} & \multicolumn{3}{c|}{
    \textbf{$\mathbb{E}_{t \geq \gamma}[\log(\frac{\hat{g}_{\theta}(\epsilon_t)}{\hat{f}_{\phi}(\epsilon_t)})]$}} \\
    \hline
    Cases 
    & $D_{KL}(f_{\phi}\|\hat{f}_{\phi})$ 
    & $D_{KL}(f_{\phi}\| \hat{g}_{\theta})$
    & sign of difference  
    & $D_{KL}(g_\theta\| \hat{f}_{\phi})$
    & $D_{KL}(g_{\theta}\| \hat{g}_{\theta})$ 
    & sign of difference  \\
    \hline
    Values & 0.3584 & 0.4044 & negative & 0.2885 & 0.0405  & positive \\
    \hline
\end{tabular}
\label{tab:kl_gripp_apollo: partial numerical}
\end{table}

\subsubsection{CUSUM with Unknown Knowledge on Post-change Distribution.}
\label{subsec: unknown}
Now, we consider the setting where no information about the post-change distribution is available for constructing an approximation of the likelihood ratio $L(\epsilon_t) = \frac{g_{\theta}(\epsilon_t)}{f_{\phi}(\epsilon_t)}$. 
We propose a robust CUSUM test that accommodates such a lack of knowledge
based on the notion of the least favorable distribution. 
%
Analysis of the actual pre- and post-change data (see Fig.~\ref{fig:fig5}) of trajectory prediction suggests that often the post-change GMM is an approximately location-shifted version of the pre-change GMM. Motivated by this, we design our robust test based on the notion of minimum shift. 
Specifically, we use the available data to learn the pre-change model 
$f_{\phi}(\epsilon_t)$. We then design the robust CUSUM test by selecting $f_{\phi}(\epsilon_t-\kappa)$ as the approximation of the post-change distribution, where $\kappa$ is a putative choice for the distributional shift. This gives us the CUSUM statistic
\begin{equation}
\label{eq:robustCUSUM_shift}
W_t^{(\kappa)} = \max\left(0, W_{t-1}^{(\kappa)} + \log \frac{f_{\phi}(\epsilon_t-\kappa)}{f_{\phi}(\epsilon_t)}\right), \quad W_0^{(\kappa)} = 0.
\end{equation}
The shift parameter $\kappa$ can be chosen either using prior knowledge or based on the standard deviation of the pre-change model. 
Below, we prove that this test is robustly optimal under certain assumptions on the actual post-change distribution.  

Define the following classes of post-change distributions:
\begin{equation}
\label{eq:postchange_models}
    \begin{split}
        \mathcal{P}_1 &= \{f_{\phi}(\epsilon_t-\eta): \eta \geq \kappa\} \\
        \mathcal{P}_2 &= \text{convex hull}(\mathcal{P}_1).
    \end{split}
\end{equation}
Thus, $\mathcal{P}_1 $ collects all possible post-change distributions that are a shifted version of the pre-change model, with the shift $\eta$ being more than the shift $\kappa$ used to design the test in \eqref{eq:robustCUSUM_shift}. The family $\mathcal{P}_2$ is the convex hull (finite convex combinations) of the distributions in $\mathcal{P}_1 $. 
Next, consider the following robust optimization formulation:
\begin{equation}
\label{eq:robust_Lorden}
    \begin{split}
        \min_\tau &\; \sup_{P \in \mathcal{P}} \text{WADD}^P(\tau)\\
        \text{subject to} & \; \text{FAR}(\tau) \leq \alpha,
    \end{split}
\end{equation}
where $\text{WADD}^P(\tau)$ is the delay when the post-change distribution is $P$. The family $\mathcal{P}$ can be either $\mathcal{P} = \mathcal{P}_1$ or $\mathcal{P} = \mathcal{P}_2$ defined in \eqref{eq:postchange_models}. We now state our main result on robust optimality.

\begin{theorem}
    The CUSUM algorithm in \eqref{eq:robustCUSUM_shift} is asymptotically robust optimal for the problem in \eqref{eq:robust_Lorden} for $\mathcal{P} = \mathcal{P}_1$ or $\mathcal{P} = \mathcal{P}_2$ as $\alpha \to 0$. 
\end{theorem}
\begin{proof}
    The minimal shift distribution $f_{\phi}(\epsilon_t-\eta)$ satisfies the weak stochastic boundedness conditions (14) and (15) given in \cite{molloy2017misspecified}.  The theorem then follows from Theorem 3 in \cite{molloy2017misspecified}. 
\end{proof}

We remark that the robust test is approximately correct, according to Definition~\ref{def:approxCorr}, for every possible post-change law in $\mathcal{P}_1$ and $\mathcal{P}_2$.

It is important to choose the minimal shift $\kappa$ carefully in \eqref{eq:robustCUSUM_shift}. In Fig.~\ref{fig:fig6_}, we compare the performance of the robust CUSUM test with other nonrobust CUSUM tests. A non-robust CUSUM test is one where the shift chosen is too large and hence affected underperforms. In the figures, $\kappa =1$, and the actual post-change corresponds to a shift of $2.5$. 

\begin{figure}[htbp] 

\captionsetup{labelformat=empty}
    \centering
    {\includegraphics[width=.7\linewidth]{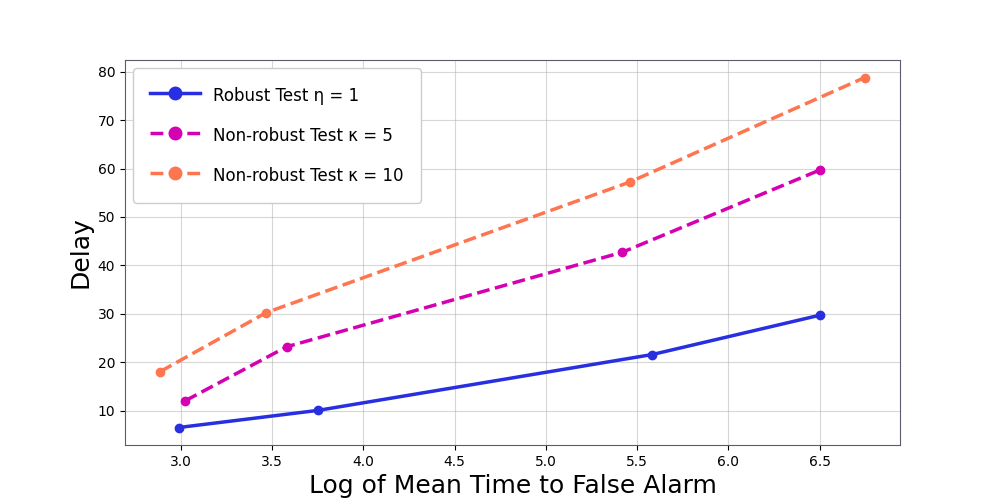}     \label{fig: subfig1}}
\caption{  {Fig. 6. }For the actual gap $\eta = 2.5$, we set the minimum test shift parameter to $\kappa = 1$ to evaluate robustness. The robust test (blue curve) with shift = 1 achieves successful detection, validating its effectiveness. In contrast, the non-robust tests ( pink curve with shift = 5 and orange curve with shift = 10) fail to maintain detection capability, as their parameter choices violate the robustness criteria.}
\Description{}
\label{fig:fig6_}
\end{figure}

\subsection{Benchmarks}
\subsubsection{Z-Score Time-series Detection}
The Z-Score method detects change points by measuring how much a data point deviates from a moving average. 
For any given window size $w$, 
define the moving average as $\bar{\epsilon}_t := \frac{1}{w} \sum_{r=t-w+1}^{t} \epsilon_r$, 
and the moving standard deviation as $\sigma_t = \sqrt{\frac{1}{w} \sum_{r=t-w+1}^{t} (\epsilon_r - \bar{\epsilon}_t)^2}$.  The Z-score of $\epsilon_t$ is 
\begin{align}
\label{eq: z-score}
z_t = \frac{\epsilon_t - \bar{\epsilon}_t}{\sigma_t}
\end{align} 

A threshold $b$ is set for the Z-Score, often based on how sensitive the detection should be. When the absolute value of the Z-Score \(|z_t|\) exceeds the threshold, the Z-Score method declares that a change has occurred. That is,  
   \[
   \tau = \min\{ t : |z_t| > b  \}.
   \]


\subsubsection{Chi-Square Time-series Detection}
Chi-Square test is commonly used to evaluate the independence of categorical variables and the fit between observed and expected frequencies. 
Specifically, we use Pearson’s chi-square test to compare observed frequencies in the time-series with expected frequencies under normal behavior \cite{franke2012chi}. Significant deviations from the expected values suggest potential anomalies. The test statistic is computed as follows: 
$ 
    \chi_t^2 = \sum_{r=t-w+1}^{t} \frac{\left( g_{\theta}(\epsilon_t) - f_{\phi}(\epsilon_t) \right)^2}{f_{\phi}(\epsilon_t)},
$
where \( w \) represents the window size. 
The null hypothesis assumes no significant difference between the frequencies. A large \(\chi^2\) value indicates an anomaly, and a change is declared when \(\chi^2\) exceeds a predefined threshold. For any given $b$, a change is declared accordingly to the following rule 
\[ 
\tau = \min\{ t : \chi_t^2 > b \}.
\]

\section{EXPERIMENTS AND RESULTS}

\subsection{Simulation Setup}
\noindent
{\bf Datasets.} 
Table \ref{tab: tab2} summarizes the characteristics of the three datasets used. The ApolloScape \cite{huang2018ApolloScape}, NGSIM \cite{fhwa2020}, and NuScenes \cite{caesar2020nuScenes} datasets are widely used in trajectory prediction and autonomous vehicle research due to their diverse and complex real-world driving scenes. ApolloScape provides a rich multimodal dataset, including camera images, LiDAR point clouds, and approximately 50 minutes of manually annotated vehicle trajectories. NGSIM focuses on freeway traffic behavior, offering detailed vehicle trajectories recorded over 45 minutes on highways US-101 and I-80. NuScenes captures 1000 diverse driving scenes in Boston and Singapore, two cities known for their challenging traffic conditions. For trajectory prediction, we choose the history trajectory length ($L_I$) and future trajectory length ($L_O$) based on the recommendations provided by the dataset’s authors. For each dataset, we randomly select 2500 scenes to serve as test cases.

\begin{table}[ht] 
\vspace{-0.3cm} 

\centering
\caption{  {Table 4.} Summary of datasets.}
\vspace{-0.3cm}
\label{tab: tab2}
\begin{center}
\begin{tabular}{|c||c||c||c||c|}
\hline
\textbf{Name} & \textbf{Scenario} & \textbf{Map} & $\mathbf{L_I}$ & $\mathbf{L_O}$ \\
\hline
ApolloScape \cite{huang2018ApolloScape} & Urban & $\times$  & 6  & 6 \\
\hline
NGSIM \cite{fhwa2020} & Highway & $\times$ & 15 & 25 \\
\hline
NuScenes \cite{caesar2020nuScenes} & Urban & \checkmark & 4  & 12 \\
\hline
\end{tabular}
\end{center}
\end{table}

\begin{table}[ht] 
\vspace{-0.3cm}
\centering
\caption{  {Table 5.} Summary of models.} 
\vspace{-0.3cm}
\label{tab: tab3}
\begin{center}
\begin{tabular}{|c||c||c||c|}
\hline
\textbf{Name} & \textbf{Input features} & \textbf{Output format} & \textbf{Network} \\
\hline
GRIP++ \cite{li2019GRIP++} & location + heading & single-prediction & Conv + GRU \\
\hline 
FQA \cite{kamra2020multi} & location  & single-prediction & LSTM \\
\hline
\end{tabular}
\end{center}
\end{table}

\noindent
{\bf Trajectory Prediction Models.}
\label{Trajectory_Prediction_Models}
Table \ref{tab: tab3} summarizes two benchmark models used for trajectory prediction. We select GRIP++ \cite{li2019GRIP++} and FQA\cite{kamra2020multi} as our trajectory prediction models due to their proven effectiveness on widely used datasets like ApolloScape and NGSIM from prior research. GRIP++ employs a graph-based structure with a two-layer GRU network, offering fast and accurate short- and long-term predictions, ranking \#1 in the 2019 ApolloScape competition, making it an ideal fit for our experiments by minimizing model noise and enhancing algorithm performance evaluation.  FQA introduces a fuzzy attention mechanism to model dynamic agent interactions, showing strong performance across diverse domains, including traffic and human motion which is well-suited for evaluating prediction performance in nuScenes dataset.

\vskip \baselineskip   

\noindent{\bf Deceptive OOD Scene Generation.}
We generate deceptive OOD scenes by introducing subtle, physically constrained changes in driving behaviors. These changes, often imperceptible to human observers, are designed to significantly degrade the prediction performance of machine learning models. Inspired by adversarial perturbation techniques \cite{zhang2022adversarial}, we apply these perturbations in a single-frame prediction setting for each dataset-model combination.  The results, presented in Table \ref{tab: tab4}, demonstrate that such perturbations consistently worsen trajectory prediction accuracy. On average, ADE increases by 148.8\%, FDE by 176.2\%, and RMSE by 50.1\%. Notably, FDE experiences the highest increase (176.2\%), highlighting greater vulnerability in long-term predictions compared to short-term or overall error measures. These findings suggest that even minor trajectory shifts could have real-world consequences if OOD scenes cannot be detected.

\begin{table}[htbp]
\caption{  {Table 6.} Average prediction error before and after the perturbation.}
\vspace{-0.3cm}
\label{tab: tab4}
\begin{center}
\begin{tabular}{|c||c||c||c||c|}
\hline
 {Model} &  {Dataset} &  {ADE (m)} &  {FDE (m)} &  {RMSE (m)} \\
\hline
 
  & ApolloScape & 1.68 / 4.15 & 2.11 / 8.73 & 6.16 / 6.26 \\
\cline{2-5}
{ GRIP++}  & NGSIM & 4.35 / 7.89 & 7.50 / 14.82 & 1.23 / 3.36 \\
\cline{2-5}
& NuScenes & 5.82 / 8.16 & 5.03 / 8.78 & 5.74 / 5.33 \\
\hline 
  & ApolloScape & 1.87 / 4.95 & 2.91 / 9.53 & 6.96 / 6.53 \\
\cline{2-5}
{ FQA} 
  & NGSIM & 4.63 / 8.69 & 7.88 / 15.62 & 1.63 / 4.16 \\
\cline{2-5}
  & NuScenes & 6.62 / 8.96 & 5.83 / 9.58 & 6.54 / 6.13 \\
\hline
\end{tabular}
\end{center}
\end{table}

\begin{figure}[t!] 
\captionsetup{labelformat=empty}
    {\includegraphics[width=.7\linewidth]{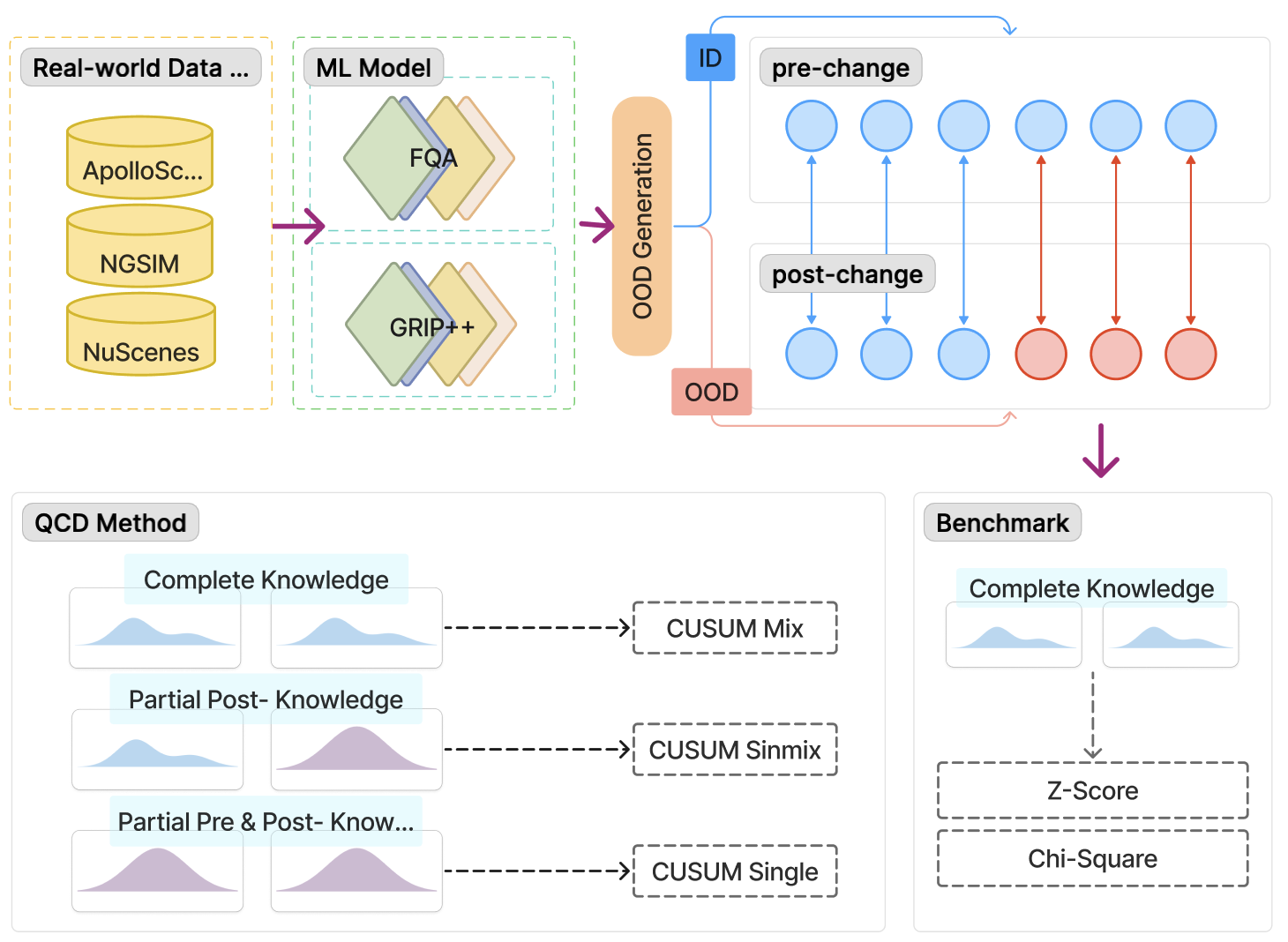}  }
\caption{  {Fig. 7. } Workflow of OOD detection in trajectory prediction. The process includes data loading, model training, OOD generation, and OOD detection using CUSUM variants and statistical benchmarks.}
\Description{}
\label{fig:workflow}
\vspace{-0.3cm}
\end{figure}

\noindent
\textbf{Workflow Design. }
Fig.~\ref{fig:workflow} illustrates our experimental workflow, which consists of four main stages: data loading, model training, OOD generation, and OOD detection. For each model-dataset combination (summarized in Tables~\ref{tab: tab2} and \ref{tab: tab3}), trajectory predictions for both ID and OOD scenarios are generated following \cite{zhang2022adversarial}. The resulting prediction error sequences are fitted into pre-change ($f_\phi$) and post-change ($g_\theta$) distributions. As discussed in Section~\ref{subsec: algoirthms}, the performance of the CUSUM algorithm varies based on the level of prior knowledge about $f_\phi$ and $g_\theta$. To address different scenarios, we employ following CUSUM variants for OOD detection and compared with benchmarks (\ie, Z-Score~\cite{cheadle2003analysis} and Chi-Square~\cite{rana2015chi}):  

\begin{itemize}
    \item \textbf{CUSUM Mix}: Assumes full knowledge of both $f_\phi$ and $g_\theta$, modeling them as Gaussian mixtures observed in real-world dataset experiments.  
    \item \textbf{CUSUM Sinmix}: Assumes full knowledge of $f_\phi$ but only partial knowledge of $g_\theta$, using Gaussian estimation to approximate the post-change distribution.  
    \item \textbf{CUSUM Single}: Assumes partial knowledge of both $f_\phi$ and $g_\theta$, modeling each as a single Gaussian distribution.
\end{itemize}




\subsection{Numerical Results}
\label{sec:num_results}
To assess the performance of our CUSUM-based change detection approach, we employ the methodology described in Section \ref{subsec: algoirthms}. Below, we present numerical results that measure the expectation values across varying levels of distributional knowledge. As illustrated in Table \ref{tab:MTFA_distribution_math_algo}, each knowledge level corresponds to a distinct algorithmic implementation. The expectation value varies depending on the integration base: for \(t < \gamma\), we integrate with respect to \(f_\phi\), whereas for \(t \geq \gamma\), we use \(g_\theta\) as the base.

\begin{table}[h]
    \caption{  {Table 7.} Representation of expectation values for corresponding algorithms under varying levels of distribution knowledge. }
    \vspace{-0.3cm}
    \centering
    \renewcommand{\arraystretch}{1.5}
    \begin{tabular}{|c|c|c|c|}
        \hline
        Distribution & Algorithm & 
        \textbf{$\mathbb{E}_{t < \gamma}[\log \hat{L}(\epsilon)]$} & 
        \textbf{$\mathbb{E}_{t \geq \gamma}[\log \hat{L}(\epsilon)]$} \\
        \hline
        Complete & CUSUM Mix & $-D_{KL}(f_{\phi}\| g_{\theta})$  
        & $D_{KL}(g_\theta\| f_{\phi})$\\
        \hline
        Partial Post & CUSUM Sinmix
         & $- D_{KL}(f_{\phi}\| \hat{g}_{\theta})$  
         & $D_{KL}(g_{\theta}(\epsilon)\| f_{\phi}(\epsilon)) - D_{KL}(g_{\theta}(\epsilon)\| \hat{g}_{\theta}(\epsilon))$\\
        \hline
        Partial Pre \& Post & CUSUM Single
        & $- D_{KL}(f_{\phi}\| \hat{g}_{\theta}) + D_{KL}(f_{\phi}\|\hat{f}_{\phi})$ & 
        $D_{KL}(g_{\theta}\| \hat{f_{\phi}}) - D_{KL}(g_{\theta}\|\hat{g}_{\theta})$\\
        \hline
    \end{tabular}
    \label{tab:MTFA_distribution_math_algo}
    \begin{minipage}{\textwidth}
\end{minipage}

\end{table}

\noindent
{\bf Expectation. }
The numerical evaluation of CUSUM-based change detection under varying distributional knowledge levels (Table~\ref{tab:tab8}) demonstrates alignment with theoretical requirements: pre-change expectations \(\mathbb{E}_{t < \gamma}[\log \hat{L}(\epsilon)]\) are strictly negative across all algorithms (\eg, CUSUM Mix: \(-1.90\) ADE for FQA), satisfying the false alarm suppression criterion, while post-change expectations \(\mathbb{E}_{t \geq \gamma}[\log \hat{L}(\epsilon)]\) remain positive (e.g., CUSUM Mix: \(1.98\) ADE for GRIP++), ensuring detection reliability. The results validate the theoretical expressions derived from KL divergences—complete knowledge (CUSUM Mix) achieves maximal post-change detectability (\(D_{KL}(g_\theta \| f_\phi)\)) but incurs higher pre-change penalties (\(-D_{KL}(f_\phi \| {g}_\theta)\)), whereas partial knowledge methods (CUSUM Sinmix/Single) exhibit lower expectation values (\eg, CUSUM Sinmix post-change expectation $< 1.09$ and CUSUM Single post-change expectation $<$ \(0.71\)).

\begin{table}[ht]
\caption{  {Table 8.} Numerical evaluation of $\mathbb{E}_{t < \gamma}[\log \hat{L}(\epsilon)]$  and $\mathbb{E}_{t \geq \gamma}[\log \hat{L}(\epsilon)]$ across GRIP++ and FQA model and three metrics (ApolloScape dataset). The outstanding result is emphasized in bold.}

\vspace{-0.2cm}
    \centering
    \begin{minipage}{0.48\textwidth}
        \centering
        \begin{tabular}{lccc|ccc}
        \toprule
        \multirow{2}{*}{\textbf{Algo}} & \multicolumn{3}{c|}{\textbf{GRIP++}} & \multicolumn{3}{c}{\textbf{FQA}} \\
        \cmidrule(lr){2-4} \cmidrule(lr){5-7}
         & ADE & FDE & RMSE & ADE & FDE & RMSE \\
        \midrule
        $\star$  Mix      & -1.69 & -1.77 & -1.73 & {\bf -1.90} & -1.28 & -2.14 \\
        $\star$  Sinmix   & -1.08 & {\bf -1.39} & -0.98 & -0.99 & -0.84 & -0.94 \\
        $\star$  Single   & -0.72 & {\bf -1.12} & -0.70 & -0.63 & -0.77 & -1.02 \\
        \bottomrule
    \end{tabular}
        \caption{(a) Expected value of pre-change ($t < \gamma$)}
        \label{tab:tab9}
    \end{minipage}
    \hfill
    \begin{minipage}{0.48\textwidth}
        \centering
        \begin{tabular}{lccc|ccc}
        \toprule
        \multirow{2}{*}{\textbf{Algo}} & \multicolumn{3}{c|}{\textbf{GRIP++}} & \multicolumn{3}{c}{\textbf{FQA}} \\
        \cmidrule(lr){2-4} \cmidrule(lr){5-7}
         & ADE & FDE & RMSE & ADE & FDE & RMSE \\
        \midrule
        $\star$  Mix      & {\bf 1.98} & 1.44 & 1.75 & 1.01 & 1.17 & 1.75 \\
        $\star$  Sinmix   & 0.95 & 0.85 & 0.98 & 0.85 & {\bf 1.09} & 0.92 \\
        $\star$  Single  & 0.54 & 0.62 & 0.70 & 0.51 & 0.66 & {\bf 0.71} \\
        \bottomrule
    \end{tabular}
        \caption{(b) Expected value of post-change (${t \geq \gamma}$)}
    \end{minipage}
      Represent CUSUM using the symbol {$^\star$} (due to space constraints).
      \label{tab:tab8}
\end{table}

\begin{table}[ht]
\caption{ {Table 9.}  Variance of of $\mathbb{E}_{t < \gamma}[\log \hat{L}(\epsilon)]$  and $\mathbb{E}_{t \geq \gamma}[\log \hat{L}(\epsilon)]$ across GRIP++ and FQA model and three metrics (ApolloScape dataset). The outstanding result is emphasized in bold.} 
\vspace{-0.2cm}
    \centering
    \begin{minipage}{0.48\textwidth}
        \centering
        \begin{tabular}{lccc|ccc}
        \toprule
        \multirow{2}{*}{\textbf{Algo}} & \multicolumn{3}{c|}{\textbf{GRIP++}} & \multicolumn{3}{c}{\textbf{FQA}} \\
        \cmidrule(lr){2-4} \cmidrule(lr){5-7}
         & ADE & FDE & RMSE & ADE & FDE & RMSE \\
        \midrule
        $^\star$ Mix    & {\bf 0.40}   & 0.42   & 0.41   & 0.43   & 0.42   & 0.41   \\
        $^\star$ Sinmix & 0.47   & {\bf 0.45}   & 0.46   & 0.53   & 0.52   & 0.54   \\
        $^\star$ Single & 0.05   & {\bf 0.04}   & 0.06   & 0.06   & 0.05   & 0.07   \\
        \bottomrule
    \end{tabular}
        \caption{(a) {Pre-change Variance} (\(\sigma^2\)) (${t < \gamma}$)}
    \end{minipage}
    \hfill
    \begin{minipage}{0.48\textwidth}
        \centering
        \begin{tabular}{lccc|ccc}
        \toprule
        \multirow{2}{*}{\textbf{Al
        go}} & \multicolumn{3}{c|}{\textbf{GRIP++}} & \multicolumn{3}{c}{\textbf{FQA}} \\
        \cmidrule(lr){2-4} \cmidrule(lr){5-7}
         & ADE & FDE & RMSE & ADE & FDE & RMSE \\
        \midrule
        {$^\star$} Mix  & 0.38 & 0.39 & 0.37 & {\bf 0.36} & 0.40 & 0.36 \\
        {$^\star$} Sinmix  & 0.88 & 0.88 & 0.81 & 0.76 & {\bf 0.75} & 0.76 \\
        {$^\star$} Single  & 0.06 & {\bf 0.04} & {\bf 0.04} & 0.05 & 0.05 & 0.06 \\
        \bottomrule
    \end{tabular}
        \caption{(b) {Post-change Variance} (\(\sigma^2\)) (${t \geq \gamma}$)}
    \end{minipage}
\vspace{-0.3cm}
\end{table}

\vskip \baselineskip  

\noindent
{\bf Variance and Stability. }
The variance analysis (Table~\ref{tab:tab9}) reveals distinct stability patterns across groups. CUSUM Mix/Sinmix shows moderate to high variances: CUSUM Mix maintains stable pre-change variances (\( \sigma^2 \approx 0.40 \text{--} 0.43 \)), slightly decreasing post-change (\( \sigma^2 \approx 0.36 \text{--} 0.40 \)), while CUSUM Sinmix exhibits higher pre-change (\( \sigma^2 \approx 0.45 \text{--} 0.54 \)) and sharply rising post-change variances (\( \sigma^2 \approx 0.75 \text{--} 0.88 \)), reflecting approximation challenges, which may also stem from inherent randomness. CUSUM Single demonstrates exceptional stability, with pre-change variances \( \sigma^2 \leq 0.07 \) and post-change \( \sigma^2 \leq 0.06 \), validating its robustness under the same distributional assumption (\ie, single Gaussian). These trends highlight the trade-off between knowledge granularity and stability, offering valuable insights into maintaining practical reliability even in the face of increased approximation uncertainty.

\subsection{CUSUM Change-point Detection}
{\bf CUSUM Detection of ID and OOD Scene. }
Fig. \ref{fig:fig7} illustrates the effectiveness of the CUSUM algorithm in detecting abnormal shifts in AV's trajectory prediction by tracking changes in prediction errors. The detection relies on monitoring the cumulative statistic \( W_t \), which sums the log-likelihood ratios between pre-change and post-change distributions. A change is declared when \( W_t \) exceeds a predefined threshold \( b \). In the OOD scene (Fig. \ref{fig:subfig7a}), a perturbation is introduced at time step 490, and we observe that \( W_t \) rapidly increases, crossing the threshold \( b = 7 \) at time step 504. This triggers a detection with a delay of only 14 samples. In contrast, in the ID scene(Fig. \ref{fig:subfig7b}), prediction errors remain stable, and \( W_t \) never exceeds the threshold $b$. Consequently, no change point is detected, confirming that the vehicle's trajectory aligns with typical patterns and that no false alarm is raised.

\begin{figure}[t!] 
\captionsetup{labelformat=empty}
    \centering
    \subfloat[OOD Scene]{\includegraphics[width=.32\linewidth]{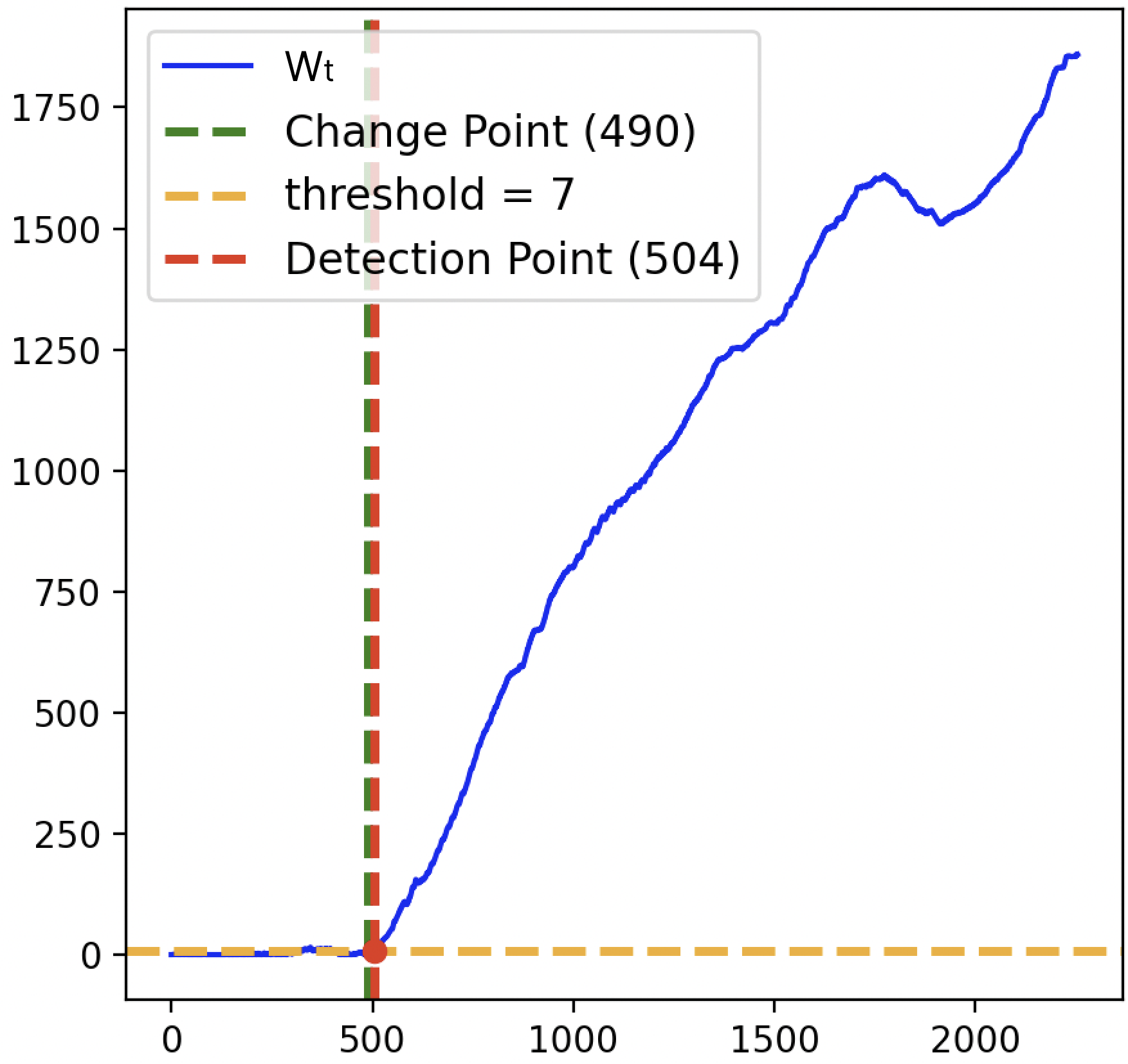}     \label{fig:subfig7a}}
    \subfloat[ID Scene]{\includegraphics[width=.3\linewidth]{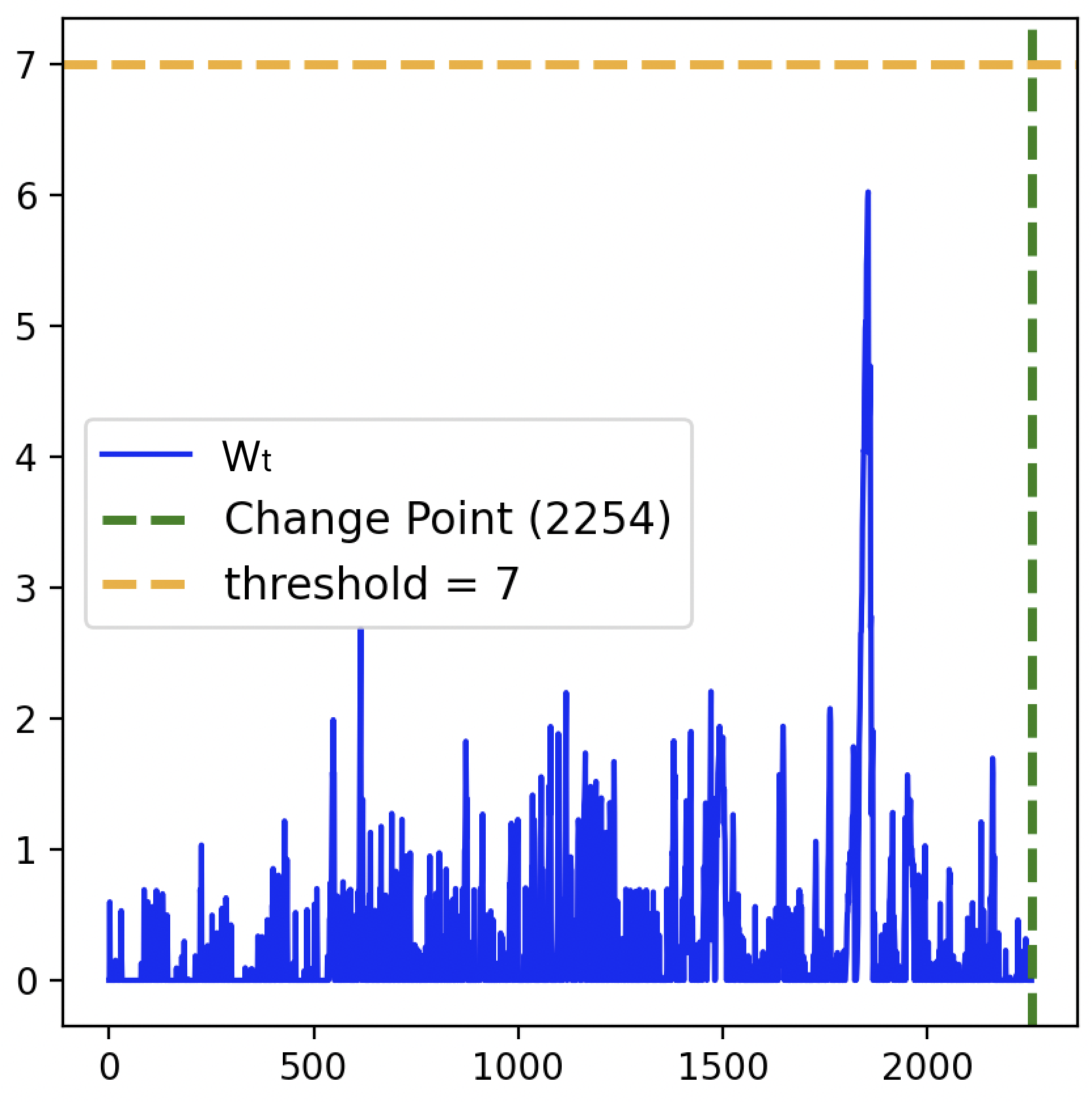}     \label{fig:subfig7b}}
\caption{  {Fig. 8. } Statistical evolution of CUSUM detection given the prediction errors from both ID and OOD scenes.
A change occurs at step 490, i.e., $\gamma = 490$.
The threshold $b$ is chosen to be 7.
At time step 504, the statistic $W_t$ surpasses the threshold ($b = 7$), declaring the detection of a change. 
The detection delay is 14 time steps. This illustrating experiment is conducted using the GRIP++ model on the ApolloScape dataset, with prediction errors evaluated using the ADE metric.}
\Description{}
\label{fig:fig7}
\end{figure}

\begin{figure*}[ht]
\centering
\captionsetup{labelformat=empty}

{\includegraphics[width=0.8\linewidth]{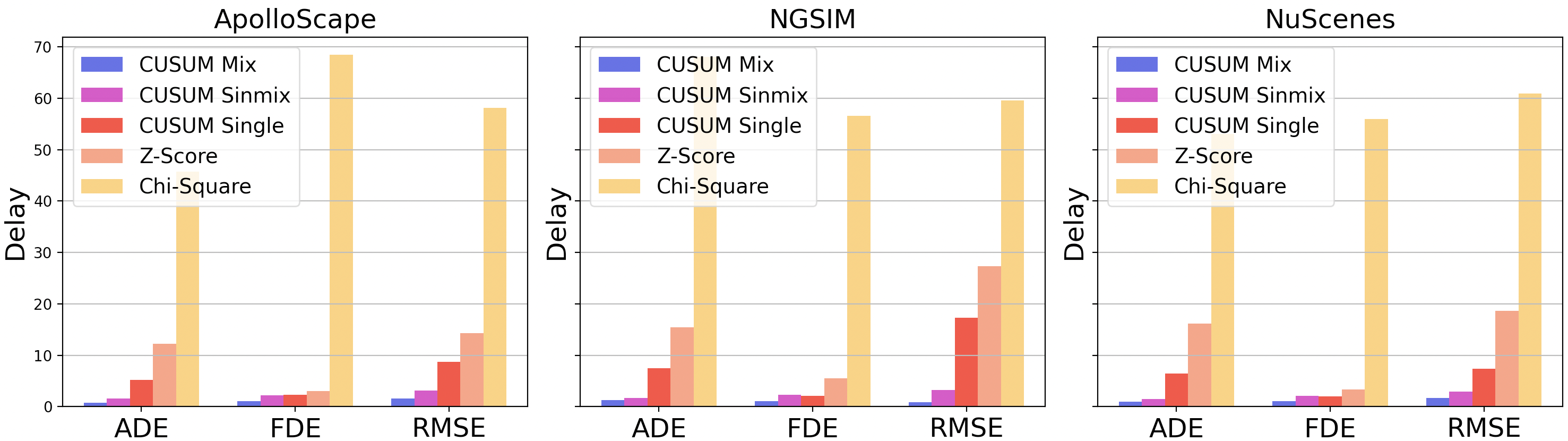}     \label{fig: subfig8a}}
    
\caption{  {Fig. 9. } WADD of the tested OOD detection algorithms across different metrics under the same MTFA. The delay is measured using \textbf{ADE}, \textbf{FDE}, and \textbf{RMSE}, providing a comparative analysis of algorithm performance in identifying anomalies. Lower delay values indicate faster detection. GRIP++ Model is used for illustration.}
\Description{}
\label{fig:fig8}
\vskip -1.3ex
\end{figure*}

\vskip \baselineskip  

\noindent
{\bf Worst-case Average Detection Delay. }
To evaluate detection delay across algorithms, all methods were standardized to identical MTFA conditions, as illustrated in Fig. \ref{fig:fig8}. Detection thresholds were systematically calibrated for each algorithm to achieve parity in MTFA performance. For rigorous statistical validation, experiments were conducted with a predefined change point at time step 0 to measure worst-case average detection delay (WADD), with 10,000 independent trials performed per method. Results revealed significant performance disparities: CUSUM Mix exhibited the shortest detection delay at 3 samples, outperforming CUSUM Sinmix (5 samples) and CUSUM Single (8 samples). In contrast, the Z-Score method required 15 samples for detection, while the baseline Chi-Square method lagged substantially at 50 samples on the ApolloScape dataset using the ADE metric. This comparative analysis underscores the superior responsiveness of CUSUM-based approaches in time-sensitive change detection scenarios, highlighting their practical advantage in applications demanding rapid anomaly identification with controlled false-alarm rates.  
\vskip \baselineskip

\noindent
{\bf False Alarm-Delay Trade-off Analysis. }
We further evaluate detection reliability across varying MTFA constraints. Fig. \ref{fig:fig9} presents a comparative evaluation of Algorithm \ref{alg: cusum_detection} applied to benchmark models GRIP++ and FQA across three datasets and three performance metrics. Across all experiments, CUSUM Mix demonstrates the most favorable trade-off, achieving the lowest detection delay while maintaining a comparable MTFA to other algorithms. CUSUM Sinmix and CUSUM Single follow, though their performance declines due to limited distributional knowledge. Nevertheless, 
cross-dataset validation (thumbnail in Fig. \ref{fig:fig9}) confirms the robustness of CUSUM-based methods. Their predictable delay scaling under dynamic thresholds—unlike the erratic behavior of Z-Score and Chi-Square—positions them as preferred solutions for real-time systems requiring simultaneous precision and responsiveness. This stability is particularly critical in applications where false-alarm constraints evolve dynamically, such as autonomous navigation or industrial monitoring.



\begin{figure*}[t!] 
    \centering
    \captionsetup{labelformat=empty}
        \includegraphics[width=\linewidth]{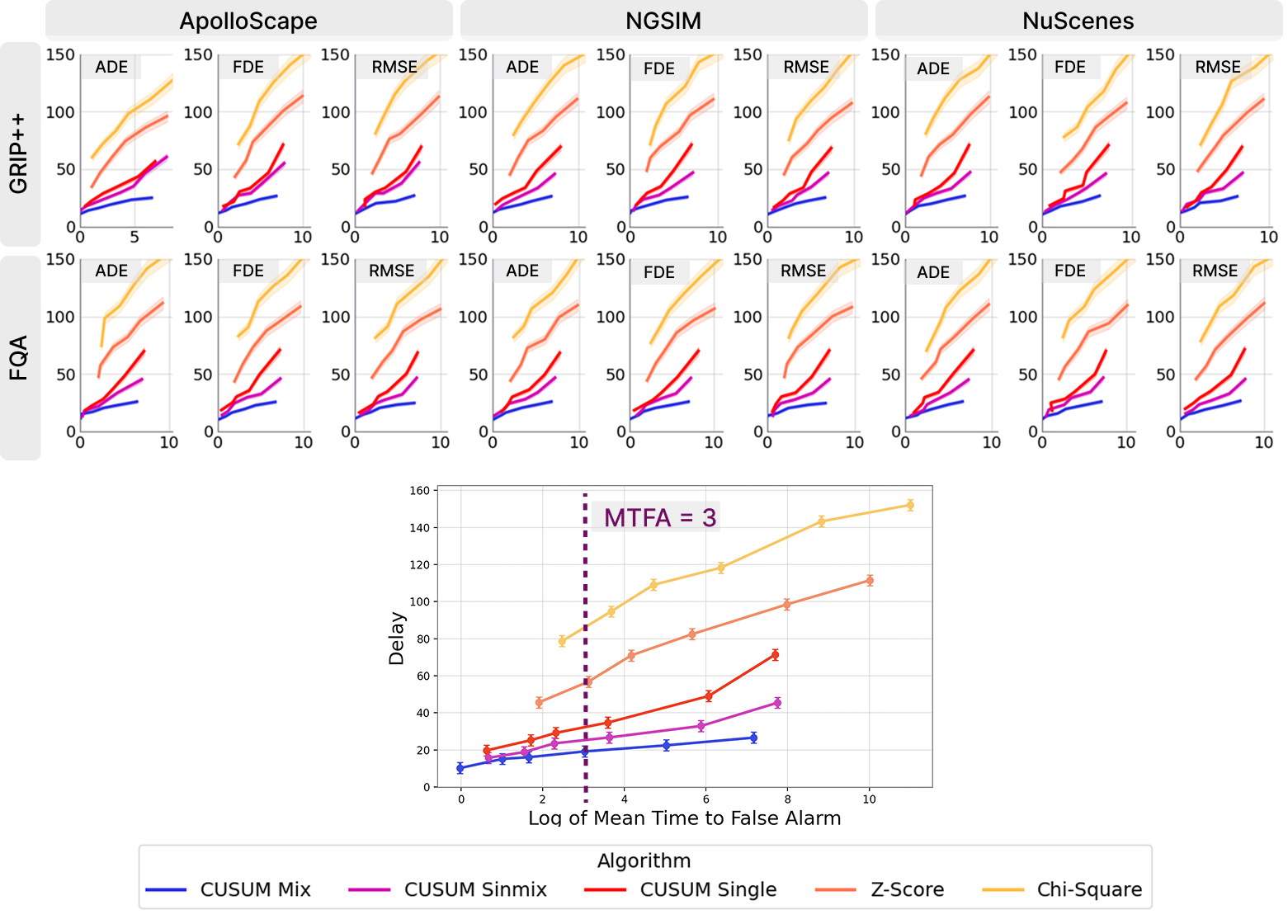}
    \caption{  {Fig. 10. } Performance of Delay over MTFA, shown for both thumbnail (overview) and zoomed-in image (detailed view). {\bf Thumbnail (overview)} compares GRIP++ (top row) and FQA (bottom row). Each row is organized in groups of three: the first three columns correspond to the ApolloScape datasets, followed by NGSIM and NuScenes. Within each group, the metrics are arranged as ADE, FDE, and RMSE, highlighting trends and ensuring consistent performance across various scenarios. {\bf Single image (detailed view)} presents an example using the GRIP++ model and the ADE metric, with the threshold \( b \) increasing from left to right.}
    \Description{}
    \label{fig:fig9}
\end{figure*}
\section{CONCLUSION}  
We propose lightweight Quick Change Detection (QCD) methods for detecting out-of-distribution (OOD) scenes in real-world trajectory prediction datasets, introducing a novel approach in this field. Our method monitors a scalar variable of prediction errors, enabling OOD detection at any point during inference. Unlike prior studies that rely heavily on simulated environments, we rigorously address this challenge using three real-world datasets—ApolloScape, NGSIM, and NuScenes—alongside two state-of-the-art prediction models, GRIP++ and FQA. Our work demonstrates the feasibility and effectiveness of applying CUSUM-based algorithms for timely and accurate OOD detection in AV systems. Experimental results show that CUSUM Mix consistently achieves superior detection performance with minimal false alarms, particularly when modeled using GMMs across all dataset-model combinations. These findings provide a robust solution for enhancing the safety and reliability of AV systems through timely model adjustments in dynamic environments. 

For future work, first, we plan to explore the development of a tiered alarm system with multiple thresholds, enabling context-aware alerts that prioritize critical warnings while filtering out benign anomalies. Such an approach could reduce computational overhead and enhance operational efficiency. Moreover, we want to further propose the adaptation solutions after the OOD scene is detected. We will explore imitation-based techniques to refine the system’s ability to mimic expert decision-making in OOD scenes. By leveraging imitation learning, the system can learn from human or expert interventions during rare or critical events, improving its ability to generalize to unseen situations. This will involve training the model on annotated datasets where expert responses to OOD events are recorded, enabling the system to better predict and respond to similar anomalies in real-time.

\renewcommand{\refname}{\MakeUppercase{References}}  
\renewcommand{\bibname}{\MakeUppercase{References}}  

\bibliographystyle{ACM-Reference-Format}
\bibliography{references}

\end{document}